\pdfoutput=1
\def\year{2021}\relax
\documentclass[letterpaper]{article} % DO NOT CHANGE THIS
\usepackage{aaai21}  % DO NOT CHANGE THIS
\usepackage{times}  % DO NOT CHANGE THIS
\usepackage{helvet} % DO NOT CHANGE THIS
\usepackage{courier}  % DO NOT CHANGE THIS
\usepackage[hyphens]{url}  % DO NOT CHANGE THIS
\usepackage{graphicx} % DO NOT CHANGE THIS
\usepackage{natbib}  % DO NOT CHANGE THIS AND DO NOT ADD ANY OPTIONS TO IT
\usepackage{caption} % DO NOT CHANGE THIS AND DO NOT ADD ANY OPTIONS TO IT
\usepackage{amssymb,amsthm, amsmath}
\usepackage{comment}
\usepackage{subcaption}
\usepackage{mathtools}
\usepackage{enumitem}
\usepackage[ruled, vlined, linesnumbered]{algorithm2e}
\usepackage{xcolor}
\usepackage{wrapfig}
\usepackage[switch]{lineno}

\usepackage{booktabs}

\theoremstyle{definition}
\newtheorem{theorem}{Theorem}
\newtheorem{lemma}{Lemma}

\newtheorem{assumption}{Assumption}

\DeclareMathOperator*{\argmax}{arg\,max}

\DeclareMathOperator{\E}{\mathbb{E}}

\title{Efficient Querying for Cooperative Probabilistic Commitments}
\author{Qi Zhang$^1$, Edmund H. Durfee$^2$, Satinder Singh$^2$\\}
\affiliations{
$^1$ Artificial Intelligence Institute, University of South Carolina \\
$^2$ Computer Science and Engineering, University of Michigan \\
qz5@cse.sc.edu, durfee@umich.edu, baveja@umich.edu
}

\begin{document}
% \linenumbers
\maketitle

\begin{abstract}
Multiagent systems can use commitments as the core of a general coordination infrastructure, supporting both cooperative and non-cooperative interactions. Agents whose objectives are aligned, and where one agent can help another achieve greater reward by sacrificing some of its own reward, should choose a cooperative commitment to maximize their \emph{joint} reward. We present a solution to the problem of how cooperative agents can efficiently find an (approximately) optimal commitment by querying about carefully-selected commitment choices. We prove structural properties of the agents' values as functions of the parameters of the commitment specification, and develop a greedy method for composing a query with provable approximation bounds, which we empirically show can find nearly optimal commitments in a fraction of the time methods that lack our insights require.
\end{abstract}

\section{Introduction}
Commitments are a proven approach to multiagent coordination~\cite{msingh,CohenLevesque,Castelfranchi1995,Mallya2003,Chesani2013,Al-Saqqar2014}. Through commitments, agents know more about what to expect from others, and thus can plan actions with higher confidence of success. That said, commitments are generally uncertain: an agent might abandon a commitment if it discovers that it cannot achieve what it promised, or that it prefers to achieve something else, or that others will not uphold their side of the commitment~\cite{Jennings,xing2001formalization,winikoff2006implementing}.

One way to deal with commitment uncertainty is to institute protocols so participating agents are aware of the status of commitments through their lifecycles~\cite{venkatraman1999verifying,xing2001formalization,yolum2002flexible,FornaraC-AAMAS08,baldoni2015composing,gunay2016promoca,pereira2017detecting,DastaniTY-JAAMAS17}. 
%ED: Moved some citations that were removed from related work to the list above.
Another has been to qualify commitments with conditional statements about what must (not) be true in the environment for the commitment to be fulfilled~\cite{msingh,Agotnes2008,Vokrinek2009}.  When such conditions might not be fully observable to all agents, agents might summarize the likelihood of the conditions being satisfied in the form of a {\em probabilistic commitment} \cite{kushmerick1994algorithm,xuan1999incorporating,Witwicki2007CommitmentdrivenDJ}.

Our focus is the process by which agents choose a probabilistic commitment, which serves as a probabilistic promise from one agent (the \emph{provider}) to another (the \emph{recipient}) about establishing a precondition for the recipient's preferred actions/objectives.  We formalize how the space of probabilistic commitments for the precondition captures different tradeoffs between timing and likelihood, where in general the recipient gets higher reward from earlier timing and/or higher likelihood, while the provider prefers later timing and/or lower likelihood because these leave it less constrained when optimizing its own policy.
Thus, when agents agree to work together (e.g., \cite{HanPL-JAAMAS17}), forming a commitment generally involves a negotiation~\cite{kraus1997negotiation,aknine2004extended,rahwan2004interest}. %where each agent prefers a commitment that maximizes its own rewards while still being acceptable to the other agent.

Sometimes, however, a pair of agents might have objectives/payoffs that are aligned/shared. For example, they might be a chef and waiter working in a restaurant (see Section~\ref{sec:Overcooked}). In a commitment-based coordination framework, such agents should find a \textit{cooperative} probabilistic commitment, whose timing and likelihood maximizes their \emph{joint} (summed) reward, in expectation.
Decomposing the joint reward into local, individual rewards is common elsewhere as well, like in the Dec-POMDP literature~\cite{oliehoek2016concise} and multi-agent reinforcement learning~\cite{zhang2018fully}.
% This optimization problem is complicated not only by being decentralized (the information relevant to optimization is distributed between the agents), but also because the space of possible timing/probability combinations is large, and the evaluation of any particular combination (requiring each agent to compute an optimal policy) is expensive.
This optimization problem is complicated by two main factors:
i) the information relevant to optimization is distributed, and thus the agents need to exchange knowledge, preferably with low communication cost;
and ii) the space of possible timing/probability combinations is large and evaluating a combination (requiring each agent to compute an optimal policy) is expensive, and thus identifying a desirable probabilistic commitment is computationally challenging even with perfect centralized information.
%As an example, consider the scenario with a chef and a waiter in a restaurant. The chef occupies the kitchen and knows the cost/time of food preparation, while the waiter occupies the dining room and knows the guests' food preferences. To coordinate, the chef can make probabilistic commitments on when and what food will be available.

% The main contribution of this paper is a decentralized, query-based approach that cooperative agents can use to efficiently converge on an approximately-optimal probabilistic commitment. To get the efficiency, we prove the existence of structural properties in the cooperating agents' value functions, and show that these can be provably exploited by our query-based approach.
The main contribution of this paper is an approach that addresses both challenges for cooperative agents to efficiently converge on an approximately-optimal probabilistic commitment.
To address i), our approach adopts a decentralized, query-based protocol for the agents to exchange knowledge effectively with low communication cost.
To get the efficiency for ii), we prove the existence of structural properties in the agents' value functions, and show that these can be provably exploited by the query-based protocol.

% The remainder of this paper is structured as follows. After situating our work in the literature (Section~\ref{sec:related-work}), we summarize the work on probabilistic commitments that we build upon (Section~\ref{sec:background}). In Section~\ref{sec:Structure of the Commitment Space}, we prove properties of the agents' value functions that allow us to safely prune portions of the possible commitment space. We then present our query-based decentralized approach (Section~\ref{sec:Commitment Queries}), and explain how the value function properties support an optimal (but more computationally costly) and an approximately-optimal (and less costly) strategy for formulating queries. We empirically evaluate the performance quality and computational costs of our new approach (Section~\ref{sec:empirical}), and discuss the promise of our approach and directions for future research (Section~\ref{sec:conclusions}).

\section{Related Work}
\label{sec:related-work}
Commitments are a widely-adopted framework for multiagent coordination \cite{kushmerick1994algorithm,xuan1999incorporating,msingh}.
We build on prior research on probabilistic commitments \cite{xuan1999incorporating,Bannazadeh2010}, where the timing and likelihood of achieving a desired outcome are explicitly specified. Choosing a probabilistic commitment thus corresponds to searching over the combinatorial space of possible commitment times and probabilities.
Prior work on such search largely relies on heuristics. 
Witwicki et al. \shortcite{Witwicki2007CommitmentdrivenDJ} propose the first probabilistic commitment search algorithm that initializes a set of commitments and then performs local adjustments on time and probability. 
Later work \cite{witwicki2009,oliehoek2012influence} further incorporates the commitment's feasibility and the best response strategy \cite{nair2003taming} to guide the search.
In contrast to these heuristic approaches, in this paper we analytically reveal the structure of the commitment space, which enables efficient search that provably finds the optimal commitment.

Because the search process is decentralized, it will involve message passing. The message passing between our decision-theoretic agents serves the purpose of preference elicitation, which is typically framed in terms of an agent querying another about which from among a set of choices it most prefers \cite{chajewska2000making,boutilier2002pomdp,viappiani2010optimal}. We adopt such a querying protocol as a means for information exchange between the agents. In particular, we draw on recent work that uses value-of-information concepts to formulate multiple-choice queries \cite{viappiani2010optimal,cohn2014characterizing,zhang2017approximately}, but as we will explain we augment prior approaches by annotating offered choices with the preferences of the agent posing the query.
Moreover, we prove several characteristic properties of agents' commitment value functions, which enables efficient formulation of near-optimal queries.

\section{Decision-Theoretic Commitments}
\label{sec:background}

The provider's and recipient's environments are modeled as two separate Markov Decision Processes (MDPs).
An MDP is defined as $M = ({S}, {A}, P, R, H, s_0)$ where ${S}$ is the finite state space, ${A}$ the finite action space, $P:{S}\times{A}\to\Delta({S})$ the transition function ($\Delta({S})$ denotes the set of all probability distributions over ${S}$), $R:{S}\times{A}\to\mathbb{R}$ the reward function, $H$ the finite horizon, and $s_0$ the initial state. The state space is partitioned into disjoint sets by the time step, ${S} = \bigcup_{h=0}^{H}{S}_h$, where states in ${S}_h$ only transition to states in ${S}_{h+1}$. The MDP starts in $s_0$ and ends in ${S}_H$. Given a policy $\pi:{S}\to \Delta({A})$, a random sequence of transitions $\{(s_h,a_h,r_{h},s_{h+1})\}_{h=0}^{H-1}$ is generated by $a_h \sim \pi(s_h), r_{h} = R(s_h,a_h), s_{h+1} \sim P(s_h, a_h)$. 
The value function of $\pi$ is $V^{\pi}_{M}(s) = \E[\textstyle\sum_{h'=h}^{H-1}r_{h'}|\pi, s_h=s]$ where $h$ is such that $s\in{S}_h$. The optimal policy $\pi^*_{M}$ maximizes $V^{\pi}_{M}$ for all $s\in{S}$, with value function $V^{\pi^*_{M}}_{M}$ abbreviated as $V^*_{M}$.

Superscripts ${\rm p}$ and ${\rm r}$ denote the provider and recipient, respectively. Thus, the provider's MDP is $M^{\rm p}$,
% $= (S^{\rm p}, A^{\rm p}, P^{\rm p}, R^{\rm p}, H^{\rm p}, s^{\rm p}_0)$
and the recipient's MDP is $M^{\rm r}$, sharing the horizon $H=H^{\rm p}=H^{\rm r}$.
% $= (S^{\rm r}, A^{\rm r}, P^{\rm r}, R^{\rm r}, H^{\rm r}, s^{\rm r}_0)$
We assume that the two MDPs are weakly-coupled in one direction in the sense that the provider's action might affect certain aspects of the recipient's state but not the other way around.
As one way to model such an interaction, we adopt the Transition-Decoupled POMDP (TD-POMDP) framework~\cite{witwicki2010influence}.
%ED: Did Becker do TD-POMDPs?  By that name?
%QI: I took it out.
Formally, both the provider's state $s^{\rm p}$ and the recipient's state $s^{\rm r}$ can be factored into state features.
The provider can fully control its state features.
The recipient's state can be factored as $s^{\rm r} =  (l^{\rm r}, u)$, where $l^{\rm r}$ is the set of all the recipient's state features {\underline{\em l}}ocally controlled by the recipient, and $u$ is the set of state features {\underline{\em u}}ncontrollable by the recipient but shared with the provider, i.e. $u = s^{\rm p} \cap s^{\rm r}$.
Formally, the dynamics of the recipient's state is factored as $P^{\rm r}=(P^{\rm r}_l, P^{\rm r}_u)$:
\begin{align*}
      P^{\rm r} \left( s^{\rm r}_{h+1}|s^{\rm r}_h,a^{\rm r}_h \right) 
    =& P^{\rm r}\left((l^{\rm r}_{h+1}, u_{h+1})|(l^{\rm r}_h, u_h),a^{\rm r}_h \right)\\
    =& P^{\rm r}_u(u_{h+1}|u_h) P^{\rm r}_l\left(l^{\rm r}_{h+1}|(l^{\rm r}_h,u_h),a^{\rm r}_h\right),
\end{align*}
where the dynamics of $u$, $P^{\rm r}_u$, is controlled only by the provider's policy (i.e., it is not a function of $a^{\rm r}_h$). 
Prior work refers to $P^{\rm r}_u$ as the {\em influence}~\cite{witwicki2010influence,oliehoek2012influence} that the provider exerts on the recipient's environment.
% Since state features $u$ are only controllable by the provider, coordination between the agents can be achieved by the provider specifying its influence from which the recipient plans accordingly~\cite{witwicki2010influence,witwicki2012heuristic,de2019influence}.
In this paper, we focus on the setting where $u$ contains a single binary state feature, $u\in\{u^-, u^+\}$, with $u$ initially taking the value of $u^-$.
Intuitively, $u^+ (u^-)$ stands for an enabled (disabled) precondition needed by the recipient, and the provider commits to enabling the precondition.
Further, we focus on a scenario where the flipping is permanent~\cite{HindriksR07,witwicki2009,zhang2016commitment}. That is, once feature $u$ flips to $u^+$, the precondition is permanently established and will not revert back to $u^-$.

\textbf{The provider's commitment semantics.}
Borrowing from the literature~\cite{Witwicki2007CommitmentdrivenDJ,zhang2016commitment}, we define a probabilistic commitment w.r.t. the shared feature $u$ via a tuple $c = (T, p)$, where $T$ is the commitment time and $p$ is the commitment probability.
The provider's commitment semantics is to follow a policy $\pi^{\rm p}$ that, starting from initial state $s^{\rm p}_0$ (in which $u$ is $u^-$), sets $u$ to $u^+$ by time step $T$ with at least probability $p$:
\begin{align}\label{eq:commitment semantics}
    \Pr \left( u^+ \in s^{\rm p}_T |s^{\rm p}_0, \pi^{\rm p} \right) \geq p.
\end{align}
For a commitment $c$, let $\Pi^{\rm p}(c)$ be the set of all possible provider policies respecting the commitment semantics (Eq.~\eqref{eq:commitment semantics}).
We call commitment $c$ {\em feasible} if and only if $\Pi^{\rm p}(c)$ is non-empty.
For a given commitment time $T$, there is a maximum feasible probability $\overline{p}(T) \le 1$ such that commitment $(T, p)$ is feasible if and only if $p\leq \overline{p}(T)$, 
and $\overline{p}(T)$ can be computed by solving the provider's MDP with the reward function modified to +1 reward for states where the commitment is realized at $T$, and 0 otherwise.
This is because maximizing this reward is equivalent to maximizing the probability of realizing the commitment at time step $T$, and thus the optimal initial state value is the maximum feasible probability $\overline{p}(T)$.

Given a feasible $c$, the provider's optimal policy maximizes the value with its original reward function of its initial state while respecting the commitment semantics:
\begin{align}\label{eq:provider commitment value}
\textstyle
    v^{\rm p}(c) = \max_{\pi^{\rm p} \in \Pi^{\rm p}(c)} V^{\pi^{\rm p}}_{M^{\rm p}} ( s^{\rm p}_0 ).
\end{align}
We call $v^{\rm p}(c)$ the provider's commitment value function, and $\pi^{\rm p}(c)$ denotes the provider's policy maximizing Eq.~\eqref{eq:provider commitment value}.

\textbf{The recipient's commitment modeling.}
Abstracting the provider's influence using a single time/probability pair reduces the complexity and communication 
% cost of the negotiation 
between the two agents,
and prior work has also shown that such abstraction, by leaving other time steps unconstrained, helps the provider handle uncertainty in its environment \cite{zhang2016commitment, zhang2020semantics}. 
Specifying just a single time/probability pair, however, increases the uncertainty of the recipient.
Given
% a 
commitment $c$, the recipient creates an approximation $\widehat{P}^{\rm r}_u(c)$ of influence $P^{\rm r}_u$, where $\widehat{P}^{\rm r}_u(c)$ hypothesizes the flipping probabilities at other timesteps. 
Formally, given $\widehat{P}^{\rm r}_u(c)$, let $\widehat{M}^{\rm r}(c)$ be the recipient's approximate model that differs from $M^{\rm r}$ only in terms of the dynamics of $u$.
The recipient's value of commitment $c$ is defined to be the optimal value of the initial state in $\widehat{M}^{\rm r}(c)$:
\begin{align}\label{eq:recipient commitment value}
\textstyle
    v^{\rm r}(c) =  \max_{\pi^{\rm r} \in \Pi^{\rm r}} V^{\pi^{\rm r}}_{\widehat{M}^{\rm r}(c)}(s^{\rm r}_0).
\end{align}
We call $v^{\rm r}(c)$ the recipient's commitment value function, and $\pi^{\rm r}(c)$ the recipient's policy maximizing Eq.~\eqref{eq:recipient commitment value} 
% for commitment $c$
.

Previous work~\cite{witwicki2010influence,zhang2020maintenance} has chosen an intuitive and straightforward strategy for the recipient to create $\widehat{P}^{\rm r}_u(c)$, which models the flipping with a single branch at the commitment time with the commitment probability.
% This strategy reduces the recipient's reasoning complexity by modeling fewer possible trajectories than considering branches at every time would.
In this paper, we adopt this commitment modeling strategy in Eq.~\eqref{eq:recipient commitment value} for the recipient, where the strategy determines the transition function of $\widehat{M}^{\rm r}(c)$ through $\widehat{P}^{\rm r}_u(c)$. 
%ED: Where is this strategy encoded in equation 3?

\textbf{The optimal commitment.}
Let $\mathcal{T} = \{1, 2,...,H\}$ be the space of possible commitment times, $[0,1]$ be the continuous commitment probability space, and $v^{\rm p+r} = v^{\rm p} + v^{\rm r}$ be the joint commitment value function.
The optimal commitment is a feasible commitment that maximizes the joint value, i.e.
\begin{align}\label{eq:optimal commitment}
\textstyle
    c^* = \argmax_{{\rm feasible}~ c\in\mathcal{T}\times[0,1]} v^{\rm p+r}(c).
\end{align}
Since commitment feasibility is a constraint for all our optimization problems, for notational simplicity we omit it for the rest of this paper.
%For notation simplicity, we will not explicitly write out this constraint for the remainder of the paper.
A na\"ive strategy for solving the problem in Eq.~\eqref{eq:optimal commitment} is to discretize the commitment probability space, and evaluate every feasible commitment in the discretized space.
The finer the discretization is, the better the solution will be. At the same time, the finer the discretization, the larger the computational cost of evaluating all the possible commitments.
Next, we prove structural properties of the provider's and the recipient's commitment value functions that
enable us to develop algorithms that {\em efficiently} search for the exact optimal commitment.

\section{Commitment Space Structure}
\label{sec:Structure of the Commitment Space}

\subsection{Properties of the Commitment Values}
We show that, as functions of the commitment probability, both commitment value functions are monotonic and piecewise linear; the provider's commitment value function is concave, and the recipient's is convex.
Proofs of all the theorems and the lemmas are included in the appendix.

% \subsection{Properties of the Provider's Value}\label{sec:The Provider's Commitment Value}

\begin{theorem}\label{theorem:provider commitment value}
Let $v^{\rm p}(c) = v^{\rm p}(T, p)$ be the provider's commitment value as defined in Eq.~\eqref{eq:provider commitment value}.
For any fixed commitment time $T$,  $v^{\rm p}(T, p)$ is monotonically non-increasing, concave, and piecewise linear in $p$.
\end{theorem}

We introduce Assumption \ref{assumption:u} that formalizes the notion that $u^+$, as opposed to $u^-$, is the value of $u$ that is desirable for the recipient, and then state the properties of the recipient's commitment value function in Theorem \ref{theorem:recipient commitment value}.
\begin{assumption}\label{assumption:u}
Let $M^{\rm r+} (M^{\rm r-})$ be defined as the recipient's MDP identical to $M^{\rm r}$ except that $u$ is always set to $u^+(u^-)$.
For any $M^{\rm r}$ and any locally-controlled feature $l^{\rm r}$, letting $s^{\rm r+} = (l^{\rm r}, u^+)$ and $s^{\rm r-} = (l^{\rm r}, u^-)$, we assume
$
    V^*_{M^{\rm r-}}( s^{\rm r-}  ) \leq V^*_{M^{\rm r+}}(s^{\rm r+}).
$
\end{assumption}

\begin{theorem} \label{theorem:recipient commitment value}
Let $v^{\rm r}(c) = v^{\rm r}(T, p)$ be the recipient's commitment value as defined in Eq.~\eqref{eq:recipient commitment value}.
For any fixed commitment time $T$,  under Assumption \ref{assumption:u}, $v^{\rm r}(T, p)$ is monotonically non-decreasing, convex, and piecewise linear in $p$.
\end{theorem}

\subsection{Efficient Optimal Commitment Search}
\label{sec:Efficient Optimal Commitment Search}
% Here we show how the structure in the recipient's and provider's value functions presented above leads to a reduced search space for optimal commitments. This, in turn, allows for an efficient \emph{centralized} search algorithm for optimal commitments that we will use to benchmark the decentralized algorithms we develop in Section \ref{sec:Commitment Queries}. 

As an immediate consequence of Theorems \ref{theorem:provider commitment value} and \ref{theorem:recipient commitment value}, the joint commitment value is piecewise linear in the probability, and any local maximum for a fixed commitment time $T$ can be attained by a probability at the extremes of zero and $\overline{p}(T)$, or where the slope of the provider's commitment value function changes.
We refer to these probabilities as the provider's {\em linearity breakpoints}, or breakpoints for short.
Therefore, one can solve the problem in Eq.~\eqref{eq:optimal commitment} to find an optimal commitment by searching only over these  breakpoints, as formally stated in Theorem \ref{theorem: breakpoints}.

\begin{theorem}\label{theorem: breakpoints}
Let $\mathcal{P}(T)$ be the provider's  breakpoints for a fixed commitment time $T$.
Let $\mathcal{C} = \{(T,p): T\in\mathcal{T}, p\in\mathcal{P}(T)\}$ be the set of commitments in which the probability is a provider's breakpoint.
We have
\begin{align*}
\textstyle
    \max_{c\in\mathcal{T}\times[0,1]} v^{\rm p+r}(c) =
    \max_{c\in\mathcal{C}} v^{\rm p+r}(c).
\end{align*}
\end{theorem}

Further, the property of convexity/concavity assures that, for any commitment time, the commitment value function is linear in a probability interval $[p_l, p_u]$ if and only if the value of an intermediate commitment probability $p_m\in(p_l,p_u)$ is the linear interpolation of the two extremes.
This enables us to adopt the binary search procedure in Algorithm \ref{algo:binary_search_for_breakpoints} to efficiently identify the provider's breakpoints.
For any fixed commitment time $T$, the strategy first computes the maximum feasible probability $\overline{p}(T)$.
Beginning with the entire interval of $[p_l,p_u]=[0, \overline{p}(T)]$, it recursively checks the linearity of an interval by checking the middle point, $p_m = (p_l+p_u)/2$. The recursion continues with the two halves, $[p_l, p_m]$ and $[p_m, p_u]$, only if the commitment value function is verified to be nonlinear in interval $[p_l,p_u]$.
Stepping through $T\in[ht]$ and doing the above binary search for each will find all probability breakpoint commitments $\mathcal{C}$.
\begin{algorithm}[t]
\caption{Binary search for breakpoints}
\label{algo:binary_search_for_breakpoints}
\SetAlgoLined
\KwIn{The provider's $M^{\rm p}$, commitment time $T$.}
\KwOut{ $\mathcal{P}(T)$: the provider's breakpoints for $T$.}
$\overline{p}(T)$ $\gets$ the maximum feasible probability for $T$\\
\texttt{q} $\gets$ A FIFO queue of probability intervals\\
 \texttt{q.push}$\left( [0, \overline{p}(T)] \right)$\\
Compute and save the provider's commitment value for $p = 0, \overline{p}(T)$, i.e. $v^{\rm p}(T,0)$ and  $v^{\rm p}(T, \overline{p}(T))$\\
Initialize $\mathcal{P}(T)\gets \{\}$\\

\While{\rm \texttt{q} not empty}{
    
    $[p_l,p_u] \gets \texttt{q.pop}()$;
    $\mathcal{P}(T)$ $\gets$ $\mathcal{P}(T)$ $\cup$ $\{p_l, p_u\}$ \\
    $p_m \gets (p_l+p_u)/2$; compute and save $v^{\rm p}(T,p_m)$\\
    \If{\rm $v^{\rm p}(T,p_m)$ is not the linear interpolation of $v^{\rm p}(T,p_l)$ and $v^{\rm p}(T,p_u)$}{
    \texttt{q.push}$\left( [p_l, p_m] \right)$;
    \texttt{q.push}$\left( [p_m, p_u] \right)$\\
    }
}
\end{algorithm}

This allows for an efficient {\em centralized} procedure to search for the optimal commitment: construct $\mathcal{C}$ as just described, compute the value of each $c\in\mathcal{C}$ for both the provider and recipient, and return the $c$ with the highest summed value.
We will use it to benchmark the decentralized algorithms we develop in Section \ref{sec:Commitment Queries}.

\section{Commitment Queries}
\label{sec:Commitment Queries}
We now develop a querying approach for eliciting the jointly-preferred (cooperative) commitment in a decentralized setting where neither agent has full knowledge about the other's environment.
In our querying approach, one agent poses a {\em commitment query} consisting of information about a set of feasible commitments, and the other responds by selecting the commitment from the set that best satisfies their joint preferences.
To limit communication cost and response time, the set of commitments in the query is often small.
A query poser thus should optimize its choices of commitments to include, and the responder's choice should reflect joint value.
In general, either the provider or recipient could be responsible for posing the query, and the other for responding, and in future work we will consider how these roles could be dynamically assigned.
In this paper, though, we always assign the provider to be the query poser and the recipient to be the responder.  We do this because the agents must assuredly be able to adopt the responder's selected choice, which means it must be feasible, and per Section~\ref{sec:background}, only the provider knows which commitments are feasible.
%the set of feasible commitments is known only to the provider.
%To maximize communication effectiveness, the poser must only offer feasible commitments, so that the agents quickly converge on the responder's selected commitment.
%Only the provider can do this as the poser.

Specifically, we consider a setting where the provider fully knows its MDP,
and where its uncertainty about the recipient's MDP is modeled as a distribution $\mu$ over a finite set of $N$ candidate MDPs containing the recipient's true MDP.
Given uncertainty $\mu$, the Expected Utility (EU) of a feasible commitment $c$ is defined as 
% the expected joint value of the commitment under $\mu$
:
\begin{align}\label{eq:Expected Utility}
    EU(c;\mu) = \E_{\mu} \left[ v^{\rm p+r}(c) \right],
\end{align}
where the expectation is w.r.t. the uncertainty about the recipient's MDP.
If the provider had to singlehandedly select a commitment based on its uncertainty $\mu$, the best commitment is the one that maximizes the expected utility:
\begin{align}\label{eq:optimal EU prior to query}
\textstyle
    c^*(\mu) = \argmax_{c} EU(c;\mu).
    % ,~\text{with } EU^*(\mu) = \max_{c} EU(c;\mu).
\end{align}
But through querying, the provider is given a chance to refine its knowledge about the recipient's actual MDP.
Formally, the provider's commitment query $\mathcal{Q}$ consists of a finite number $k=|\mathcal{Q}|$ of feasible commitments. The provider offers these choices to the recipient, where the provider also annotates each choice with the expected local value of its optimal policy respecting the commitment (Eq.~\eqref{eq:provider commitment value}). 
The recipient computes (using Eq.~\eqref{eq:recipient commitment value}) its own expected value for each commitment offered in the query, and adds that to the annotated value from the provider. 
It responds with the commitment that maximizes the summed value (with ties broken by selecting the smallest indexed) to be the commitment the two agents agree on.
Therefore, our motivation for a small query size $k$ is two-fold: it avoids large communication cost; and it induces short response time of the recipient evaluating each commitment in the query.

More formally, let $\mathcal{Q} \rightsquigarrow c$ denote the recipient's response that selects $c\in\mathcal{Q}$.
With the provider's prior uncertainty $\mu$, the posterior distribution given the response is denoted as $\mu~|~\mathcal{Q} \rightsquigarrow c$, which can be computed by Bayes' rule.
When the query size $k=|\mathcal{Q}|$ is limited, the response usually cannot fully resolve the provider's uncertainty.
In that case, the value of a query $\mathcal{Q}$ is the EU with respect to the posterior distribution averaged over all the commitments in the query being a possible response,
and, consistent with prior work~\cite{viappiani2010optimal},
% ,zhang2017approximately
we refer to it as the query's Expected Utility of Selection (EUS):
\begin{align*}
    EUS(\mathcal{Q};\mu) = \E_{\mathcal{Q}\rightsquigarrow c;\mu} \left[ EU(c; \mu~|~\mathcal{Q}\rightsquigarrow c) \right]. 
\end{align*}
Here, the expectation is with respect to the recipient's response under $\mu$.
The provider's \emph{querying problem} thus is to formulate a query $\mathcal{Q}\subseteq\mathcal{T} \times [0,1]$ consisting of $|\mathcal{Q}|=k$ feasible commitments that maximizes EUS:
\begin{align}\label{eq:EUS maximization}
\textstyle
    \max_{\mathcal{Q} \subseteq \mathcal{T}\times[0,1], |\mathcal{Q}|=k} EUS(\mathcal{Q};\mu).
\end{align}

Importantly, we can show that $EUS(\mathcal{Q};\mu)$ is a submodular function of $\mathcal{Q}$, as formally stated in Theorem \ref{theorem: submodularity}.
Submodularity serves as the basis for a greedy optimization algorithm~\cite{nemhauser1978analysis}, which we will describe after Theorem \ref{theorem: optimal discretization}.
\begin{theorem}\label{theorem: submodularity}
For any uncertainty $\mu$, $EUS(\mathcal{Q};\mu)$ is a submodular function of $\mathcal{Q}$. 
% That is, given two queries $\mathcal{Q}\subseteq \mathcal{Q}'$, commitment $c\notin\mathcal{Q}$, we have:
% \begin{align*}
%     EUS(\mathcal{Q} \cup \{c\};\mu) - EUS(\mathcal{Q};\mu) 
%     \geq
%     EUS(\mathcal{Q}' \cup \{c\};\mu) - EUS(\mathcal{Q}';\mu)
% \end{align*}
\end{theorem}
Submodularity means that adding a commitment to the query can increase the EUS, but the increase is diminishing with the size of the query. 
An upper bound on the EUS of any query of any size $k$ can be obtained when $k\geq N$ such that the query can include the optimal commitment of each candidate recipient's MDP, i.e. 
\begin{align}\label{eq:EUS upper bound}
\textstyle
    \overline{EUS} = \E_{\mu} \left[ \max_{c\in\mathcal{T}\times[0,1]} v^{\rm p+r}(c) \right ].
\end{align}
% Upper bound $\overline{EUS}$ can be computed with the centralized algorithm we described in Section \ref{sec:Efficient Optimal Commitment Search}.
As the objective of Eq. \eqref{eq:EUS maximization} increases with size $k$, in practice the agents could choose size $k$ large enough to meet some predefined EUS.  
We will empirically investigate the effect of the choice of $k$ in Section \ref{sec:empirical}.

\paragraph{Structure of the Commitment Query Space.}
% \label{sec:Structure of the Commitment Query Space}
%ED: Technically, we shouldn't have a section 5.1 without a section 5.2....
Due to the properties of individual commitment value functions proved in Section \ref{sec:Structure of the Commitment Space}, the expected utility $EU(c; \mu)$ defined in Eq.~\eqref{eq:Expected Utility}, as calculated by the provider alone, becomes a summation of the non-increasing provider's commitment value function and the (provider-computed) weighted average of the non-decreasing recipient's commitment value functions.
With the same reasoning as for Theorem \ref{theorem: breakpoints}, the optimality of the breakpoint commitments can be generalized to any uncertainty, as formalized in Lemma \ref{lemma: breakpoints}.

\begin{lemma}\label{lemma: breakpoints}
Let $\mathcal{C}$ be defined as in Theorem \ref{theorem: breakpoints}. We have
$
    \max_{c\in\mathcal{T}\times[0,1]} EU(c;\mu) =
    \max_{c\in\mathcal{C}} EU(c;\mu).
$
\end{lemma}

As a consequence of Lemma \ref{lemma: breakpoints}, for EUS maximization, there is no loss in only considering the provider's breakpoints, as formally stated in Theorem \ref{theorem: optimal discretization}.

\begin{theorem}\label{theorem: optimal discretization}
% Let $\mathcal{C}$ be defined in the same manner as in Theorem \ref{theorem: breakpoints}.
For any query size $k$ and uncertainty $\mu$, we have
\begin{align*}
% \textstyle
  \max_{\mathcal{Q}\subseteq \mathcal{T}\times [0,1], |\mathcal{Q}|=k} EUS(\mathcal{Q};\mu) 
  =
  \max_{\mathcal{Q}\subseteq \mathcal{C}, |\mathcal{Q}|=k} EUS(\mathcal{Q};\mu).
\end{align*}
\end{theorem}
% \subsection{Efficient Commitment Query Formulation}
% \label{sec:Efficient Commitment Query Formulation}
Theorem \ref{theorem: optimal discretization} enables
% allows us to develop 
an efficient procedure for solving the query formulation problem (Eq.~\eqref{eq:EUS maximization}).
The provider first identifies its breakpoint commitments $\mathcal{C}$ and evaluates them for its MDP and each of the $N$ recipient's possible MDPs.
Due to the concavity and convexity properties, 
% commitments 
$\mathcal{C}$ can be identified and evaluated efficiently with the binary search strategy we described in Section~\ref{sec:Efficient Optimal Commitment Search}.
Finally, a size $k$ query is formulated from commitments $\mathcal{C}$ that solves the EUS maximization problem either exactly with exhaustive search, 
or approximately with greedy search \cite{viappiani2010optimal,cohn2014characterizing}. The greedy search begins with $\mathcal{Q}_0$ as an empty set and iteratively performs $\mathcal{Q}_{i}\leftarrow \mathcal{Q}_{i-1} \cup \{c_i\}$ for $i=1,...,k$, where
$
c_i =
\argmax_{c\in\mathcal{C}, c\notin\mathcal{Q}_{i-1}} 
EUS(\mathcal{Q}_{i-1} \cup \{c\}; \mu).
$
Since EUS is a submodular function of the query (Theorem~\ref{theorem: submodularity}), the greedily-formed size $k$ query $\mathcal{Q}_k$ is within a factor of $1-(\frac{k-1}{k})^k$ of the optimal EUS~\cite{nemhauser1978analysis}.
% \begin{description}
% \item [Exhaustive query search.]
% The finite EUS maximization problem can be exactly solved by exhaustively forming and evaluating each $k$-subset of breakpoint commitments, and selecting the best one.

% \item [Greedy query search.]
% The finite EUS maximization problem can be approximately solved by a greedy procedure~\cite{viappiani2010optimal,cohn2014characterizing} that iteratively grows the query by adding the breakpoint commitment that contributes maximum EUS.
% Formally, beginning with $\mathcal{Q}_0$ as an empty set, the algorithm iteratively performs $\mathcal{Q}_{i}\leftarrow \mathcal{Q}_{i-1} \cup \{c_i\}$ for $i=1,...,k$, where
% $
% c_i =
% \argmax_{c\in\mathcal{C}, c\notin\mathcal{Q}_{i-1}} 
% EUS(\mathcal{Q}_{i-1} \cup \{c\}; \mu).
% $
% Since EUS is a submodular function of the query (Theorem~\ref{theorem: submodularity}), the greedily formed size $k$ query $\mathcal{Q}_k$ is within a factor of $1-(\frac{k-1}{k})^k$ of the optimal EUS~\cite{nemhauser1978analysis}.
% \end{description}

\section{Empirical Evaluation}
\label{sec:empirical}
Our empirical evaluations focus on these questions:
\begin{itemize}
    \item For EUS maximization, how effective and efficient is the breakpoints discretization compared with alternatives?
    \item For EUS maximization, how effective and efficient is greedy query search compared with exhaustive search? 
    % \item By maximizing EUS, does the commitment querying process derive high-quality joint policies for coordination?
\end{itemize}
To answer these questions, in Section \ref{sec:Synthetic MDPs}, we conduct empirical evaluations in synthetic MDPs with minimal assumptions on the structure of transition and reward functions, and we use an environment in Section \ref{sec:Overcooked} inspired by the video game of Overcooked to evaluate the breakpoints discretization and the greedy query search in this more grounded and structured domain. 
% with structured transition and reward functions

\subsection{Synthetic MDPs}
\label{sec:Synthetic MDPs}
The provider's environment is a randomly-generated MDP. It has 10 states the provider can be in at any time step, one of which is an absorbing state denoted as $s^+$, and where the initial state is chosen from the non-absorbing states.
Feature $u$ takes the value of $u^+$ only in the absorbing state, i.e. $u^+ \in s^{\rm p}$ if and only if $s^{\rm p} = s^+$.
There are 3 actions.
For each state-action pair $(s^{\rm p},a^{\rm p})$ where $s^{\rm p}\neq s^+$, the transition function $P^{\rm p}(\cdot|s^{\rm p},a^{\rm p})$ is determined independently by filling the 10 entries with values uniformly drawn from $[0, 1]$, and normalizing $P^{\rm p}(\cdot|s^{\rm p},a^{\rm p})$.
The reward $R^{\rm p}(s^{\rm p}, a^{\rm p})$ for a non-absorbing state $s^{\rm p}\neq s^+$ is sampled uniformly and independently from $[0, 1]$, and for the absorbing state $s^{\rm p}= s^+$ is zero. Thus, the random MDPs are intentionally generated to introduce a tension for the provider between helping the recipient (but getting no further local reward) versus accumulating more local reward. (Our algorithms also work fine in cases without this tension, but the commitment search is less interesting without it because no compromise is needed.)
%ED: Revised the above

The recipient's environment
% , inspired by the random walk domains used in the planning literature~\cite{fern2004learning},
is a one-dimensional space with $10$ locations represented as integers $\{0,1,...,9\}$.
In locations $1-8$, the recipient can move right, left, or stay still.
Once the recipient reaches either end (location $0$ or $9$), it stays there.
There is a gate between locations $0$ and $1$ for which $u=u^+$ denotes the state of open and $u=u^-$ closed.
Initially, the gate is closed and the recipient starts at an initial location $L_0$.
A negative reward of $-10$ is incurred by bumping into the closed gate.
For each time step the recipient is at neither end, it gets a reward of $-1$.
If it reaches the left end (i.e. location 0), it gets a one-time reward of $r_{0}>0$.
The recipient gets a reward of 0 if it reaches the right end.
In a specific instantiation,
% of the recipient's MDP
$L_0$ and $r_{0}$ are fixed.
$L_0$ is randomly chosen from locations $1-8$ and $r_{0}$ from interval $(0,10)$ to create various MDPs for the recipient.

To generate a random coordination problem, we sample an MDP for the provider, and $N$ candidate MDPs for the recipient, setting the provider's prior uncertainty $\mu$ over the recipient's MDP to be the uniform distribution over the $N$ candidates.
The horizon for both agents is set to be $20$.
Since the left end has higher rewards than the right end, if the recipient's start position is close enough to the left end and the provider commits to opening the gate early enough with high enough probability, the recipient should utilize the commitment by checking if the gate is open by the commitment time, and pass through it if so; 
otherwise, the recipient should simply ignore the commitment and move to the right end.
The distribution for generating the recipient's MDPs is designed to include diverse preferences regarding the commitments, such that the provider's query should be carefully formulated to elicit the recipient's preference.

\paragraph{Evaluating the Breakpoints Discretization.}
The principal result from Section \ref{sec:Structure of the Commitment Space}
% , where we proved properties of the agents' value functions, 
was that the commitment probabilities to consider can be restricted to breakpoints without loss of optimality. Further, the hypothesis was that the space of breakpoints would be relatively small, allowing the search to be faster.  We now empirically confirm the optimality result, and test the hypothesis of greater efficiency, by comparing the breakpoint commitments discretization to the following alternative discretizations:
\begin{description}
\item \textit{Even discretization.}
%ED: Really could use a citation or two below for which prior work uses the even discretization.
Prior work 
% using heuristic search
~\cite{Witwicki2007CommitmentdrivenDJ}
% described in Section \ref{sec:related-work} 
discretizes the probability space up to a certain granularity.
Here, the probability space $[0,1]$ is evenly discretized as  $\{p_0, ..., p_n \}$ where $p_i =\frac{i}{n}$.
\item 
\textit{Deterministic Policy (DP) discretization.}
This discretization finds all of the probabilities of toggling feature $u$ at the commitment time that can be attained by the provider following a deterministic policy~\cite{Witwicki2007CommitmentdrivenDJ,witwicki2010influence}. 
\end{description}
For the even discretization, we consider the resolutions $n\in\{10, 20, 50\}$.
For DP, we found that the number of toggling probabilities of all the provider's deterministic policies is large, and the corresponding computational cost of identifying and evaluating them is high.
To reduce the computational cost and for fair comparison, we group the probabilities in the DP discretization that
are within $\frac{i}{n}$ of each other for $n\in\{10, 20, 50\}$.
Since the problem instances have different reward scales, to facilitate analyses we normalize for each instance the EUS with the upper bound $\overline{EUS}$ defined in Eq.~\eqref{eq:EUS upper bound} and the EUS of the optimal and greedy query of the even discretization for $k=1, n=10$.
% , which we denote as $\underline{EUS}$.
% That is, a value for EUS is normalized as $(EUS-\underline{EUS} ) / (\overline{EUS}-\underline{EUS})$.
\begin{figure}[tb]
\centering
\begin{subfigure}{.23\textwidth}
  \centering
  \includegraphics[width=\linewidth]{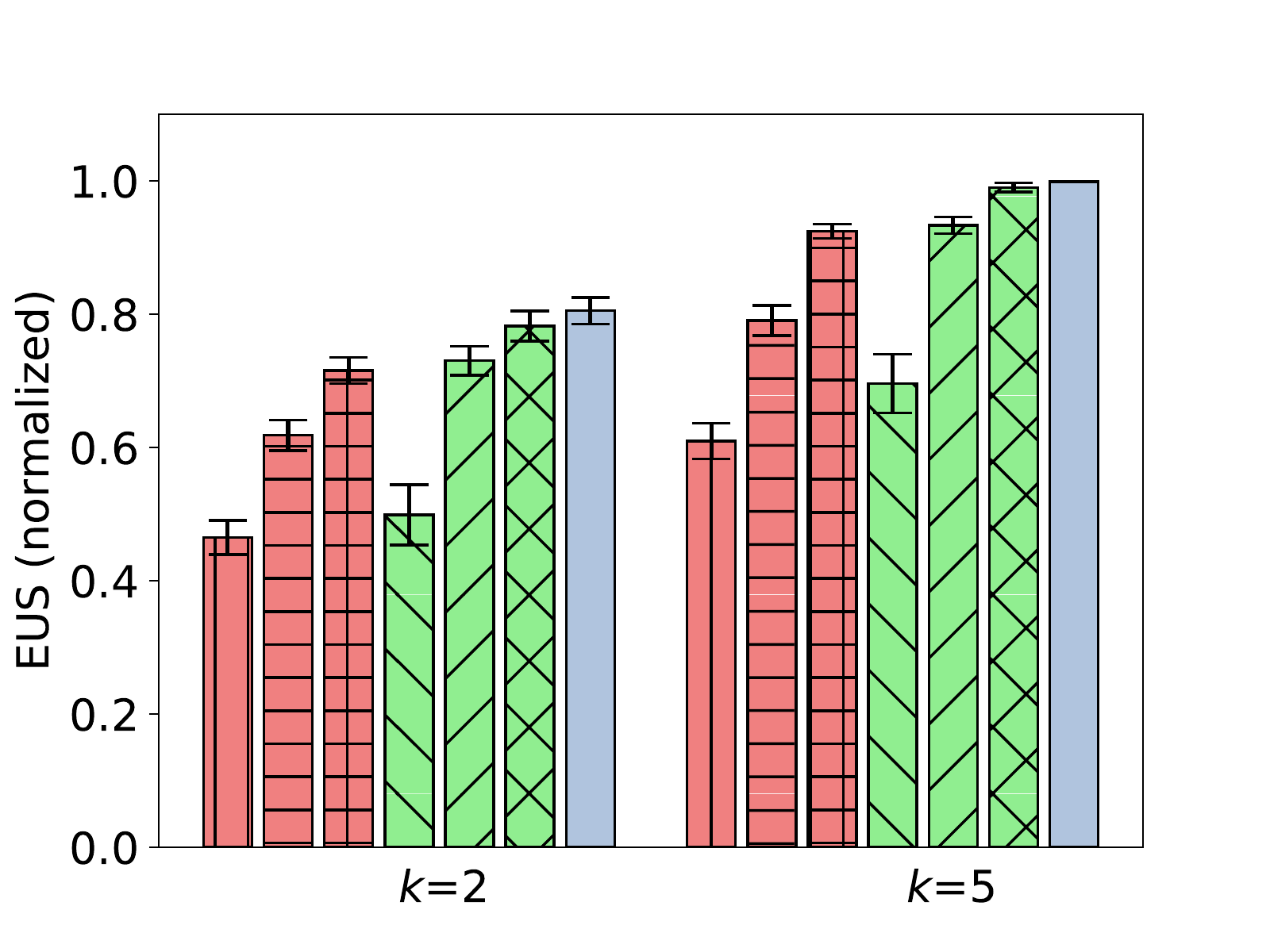}
\end{subfigure}
\begin{subfigure}{.21\textwidth}
  \centering
  \includegraphics[width=\linewidth]{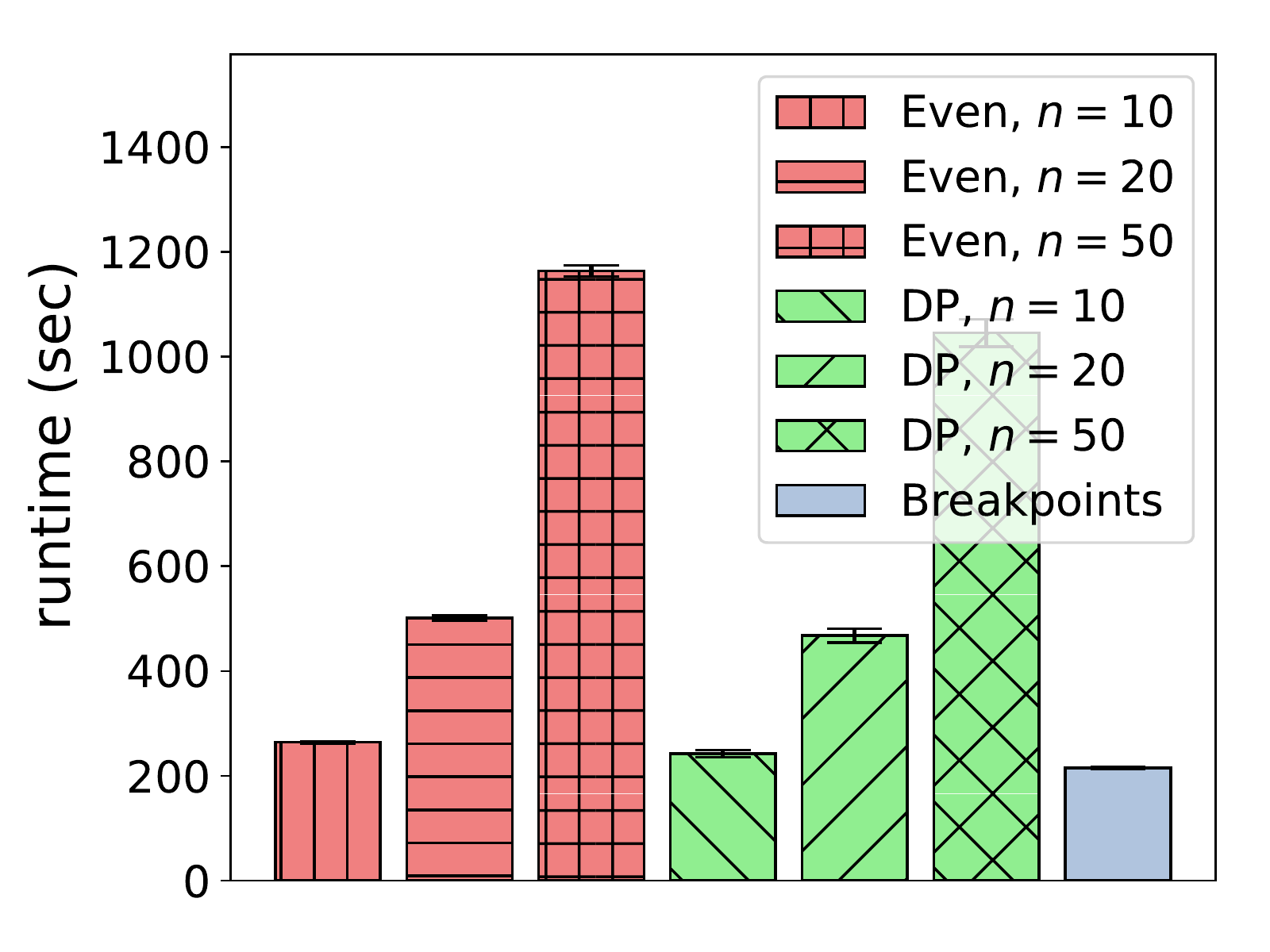}
\end{subfigure}
\caption{Means and standard errors of the EUS (left) and runtime (right) of the  discretizations in Synthetic MDPs.}
% \vspace{-3mm}
\label{fig:EUS_runtime_vs_baselines}
\end{figure}

\begin{table}[t]
\centering
\caption{Averaged discretization size per commitment time (mean and standard error) in Synthetic MDPs.}
\label{table:discretization size}
\begin{tabular}{c|ccc} 
\toprule
            & $n=10$                      & $n=20$                      & $n=50$                       \\
\hline
Even        & $7.8 \pm 0.1$ & $15.1\pm0.1$ & $37.1\pm0.1$  \\
DP          & $6.1 \pm 0.2$ & $12.0\pm0.3$ & $26.5\pm0.7$  \\
\hline
Breakpoints & \multicolumn{3}{c}{$10.0\pm0.1$}                                       \\
\bottomrule
\end{tabular}
\end{table}

Figure \ref{fig:EUS_runtime_vs_baselines} gives the EUS for the seven discretizations over $50$ randomly-generated problem instances, for $N=10$ candidate MDPs for the recipient and $k=2$ and $5$.
Figure \ref{fig:EUS_runtime_vs_baselines} shows that, coupled with the greedy query algorithm, our breakpoint commitments discretization yields the highest EUS with the lowest computational cost.
In Figure \ref{fig:EUS_runtime_vs_baselines}(left), we see that, for the even and the DP discretizations, the EUS increases with the probability resolution $n$, and only once we reach $n=50$ is the EUS comparable to our breakpoints discretization.
Figure \ref{fig:EUS_runtime_vs_baselines}(right) compares the runtimes of forming the discretization and evaluating the commitments in the discretization for the downstream query formulation procedure, confirming the hypothesis that using breakpoints is faster.
Table \ref{table:discretization size} compares the sizes of these discretizations, and confirms our intuition that the breakpoints discretization is most 
% computationally 
efficient because it identifies fewer commitments that are sufficient for the EUS maximization.

\paragraph{Evaluating the Greedy Query.}
Next, we empirically confirm that the greedy query search is effective for EUS maximization. Given the results confirming the effectiveness and efficiency of the breakpoint discretization, the query searches here are over the breakpoint commitments.
% Specifically, we show that, given the breakpoint commitments, formulating the commitment query greedily yields EUS that is comparable to the optimal, and is computationally much more efficient than exhaustively searching for the optimal query.
Figure \ref{fig:EUS_runtime_vs_k}(left) compares the EUS of the greedily-formulated query with the optimal (exhaustive search) query, and with a query comprised of randomly-chosen breakpoints. 
The EUS is normalized with $\overline{EUS}$ and the optimal EU prior to querying given uncertainty $\mu$ as defined in Eq.~\eqref{eq:optimal EU prior to query}. 
% That is, a value for EUS is normalized as $ (EUS-EU^*(\mu) ) / (\overline{EUS}-EU^*(\mu))$. (Note that $EU^*(\mu)$ is also the EUS of the optimal and greedy query when $k=1$, since the recipient is only given one choice, which is the one optimizing the provider's model.)
We vary the query size $k$, and report means and standard errors over the same $50$ coordination problems.
We see that the EUS of the greedy query tracks that of the optimal query closely, while greedy's runtime scales much better.

\begin{figure}[t]
\centering
\begin{subfigure}{.23\textwidth}
  \centering
  \includegraphics[width=\linewidth]{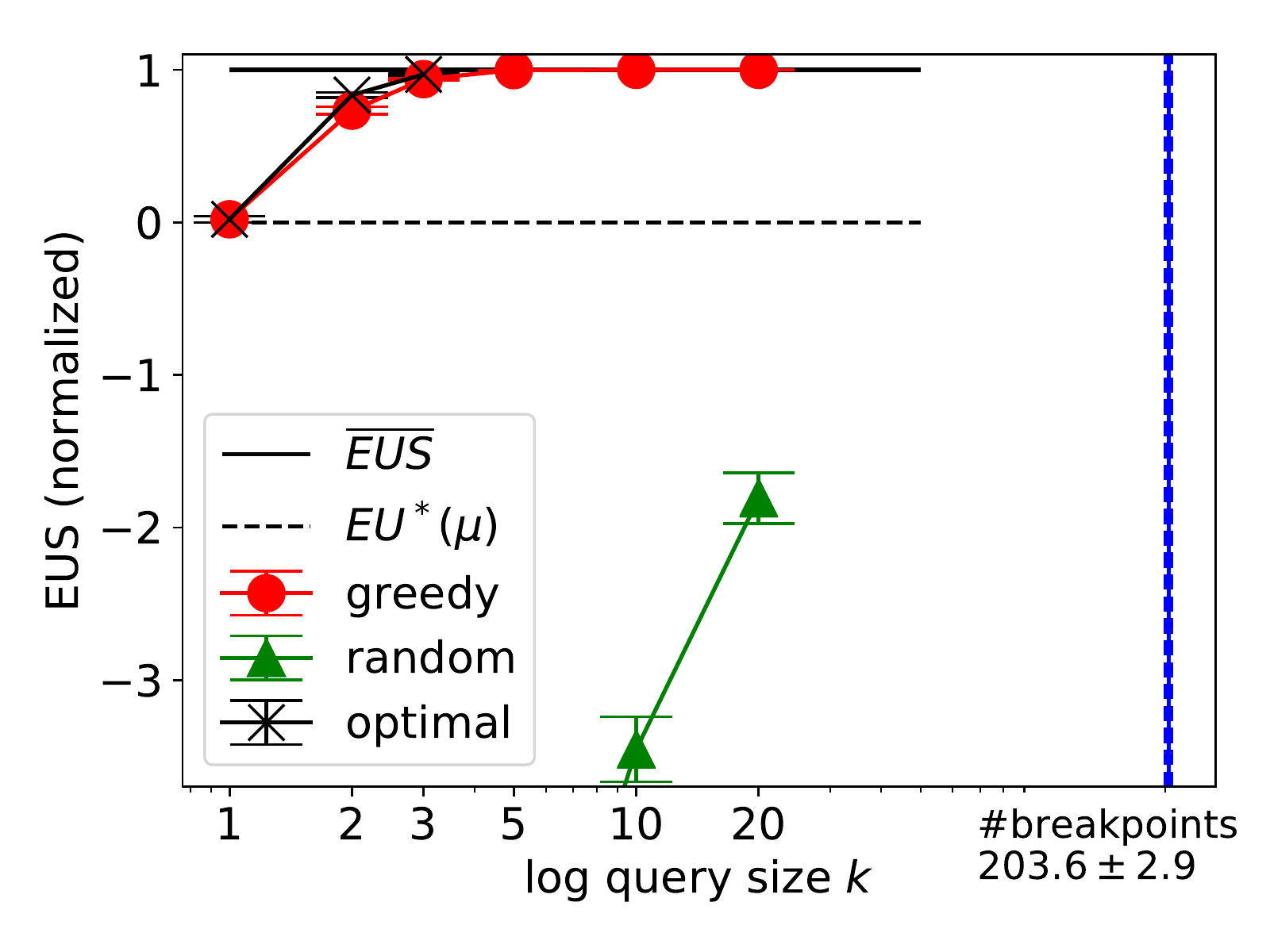}
\end{subfigure}
\begin{subfigure}{.23\textwidth}
  \centering
  \includegraphics[width=\linewidth]{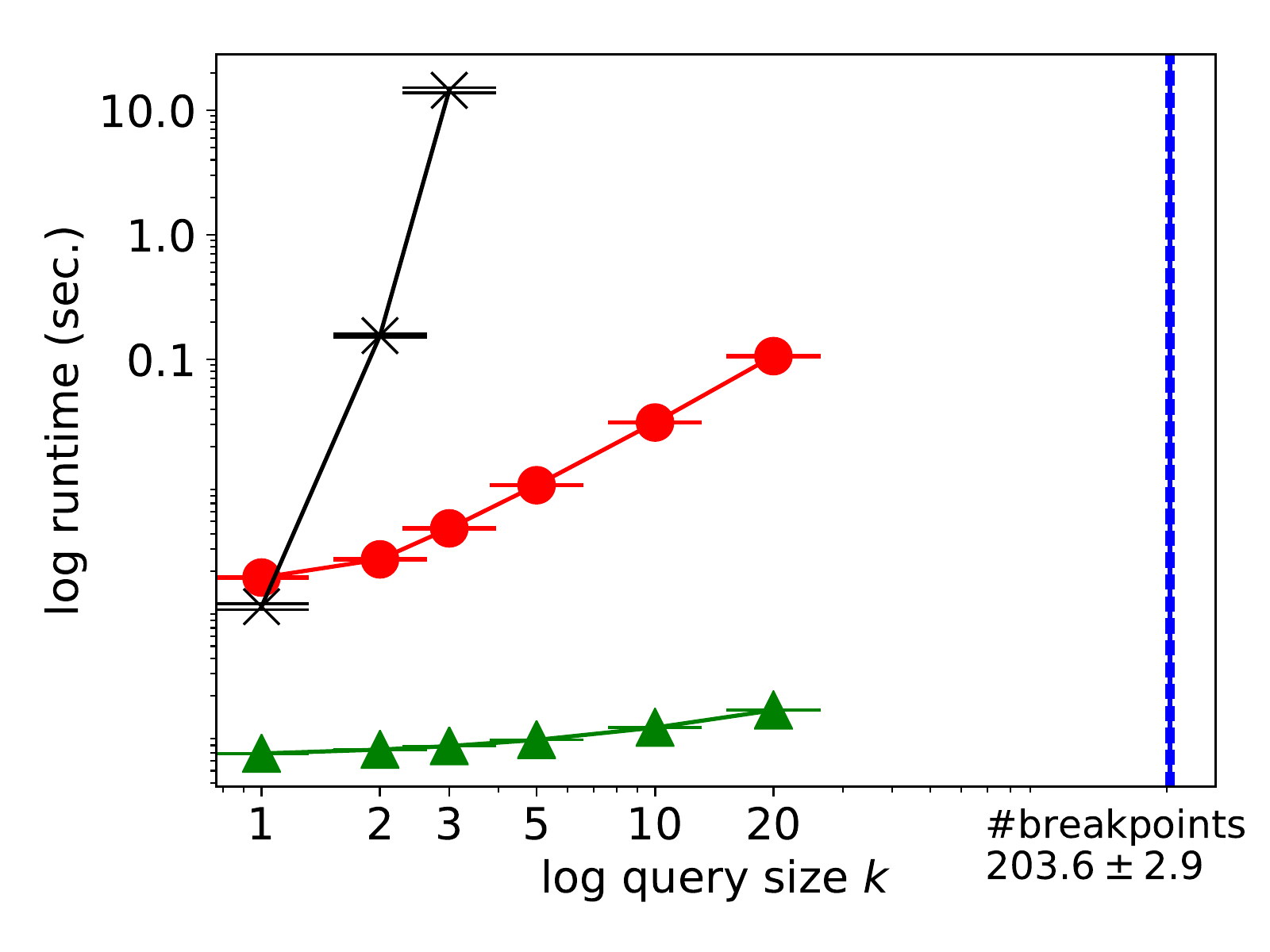}
\end{subfigure}
\caption{Means and standard errors of the EUS (left) and runtime (right) of the optimal, the greedy, and the random queries formulated from the breakpoints in Synthetic MDPs.
}
% \vspace{-3mm}
\label{fig:EUS_runtime_vs_k}
\end{figure}

% \paragraph{Diverse priors and multi-round querying.}
\begin{comment}
We have also considered larger numbers of the recipient's candidate MDPs, and various types of the provider's prior $\mu$ besides the uniform prior, such as Gaussian priors.
Besides priors that are synthetically generated, we have also explored priors that naturally emerge in a multi-round querying process.
We observe consistent results in these various settings, confirming the robustness of greedy querying. 
\textcolor{blue}{
We provide details in Appendix \ref{appendix:Diverse Priors and Multi-round Querying}.
}
\end{comment}

\subsection{Overcooked}
\label{sec:Overcooked}
In this section, we further test our approach in a more grounded domain.
The domain, Overcooked, was inspired by the video game of the same name and introduced by \cite{wang2020too} to study theory of mind in the absence of communication but with global observability. 
We reuse one of their Overcooked settings with two high-level modifications:
1) instead of having global observability, each agent observes only its local environment,
and 2) we introduce probabilistic transitions.
These modifications induce for the domain a rich space of meaningful commitments, over which the agents should carefully negotiate for the optimal cooperative behavior.
\begin{wrapfigure}{r}{.217\textwidth}
% \begin{figure}
\begin{center}
    \includegraphics[width=\linewidth]{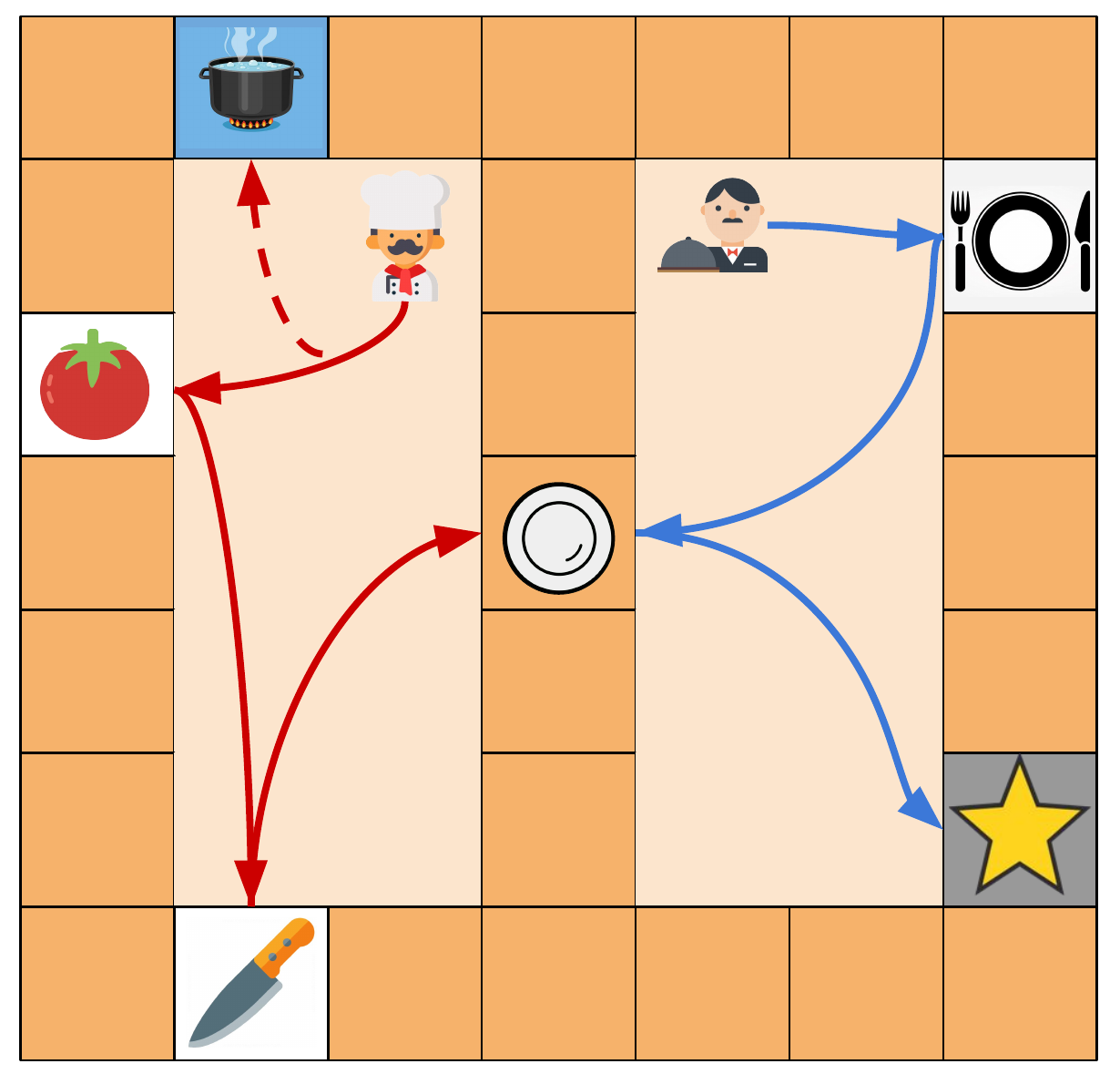}
\end{center}
\caption{Overcooked.}
% \vspace{-3mm}
\label{fig:overcooked}
% \end{figure}
\end{wrapfigure}
Figure \ref{fig:overcooked} illustrates this Overcooked environment.
Two agents, the chef and the waiter, together occupy a grid with counters being the boundaries.
The chef is supposed to pick up the tomato, chop it, and place it on the plate. Afterwards, the waiter is supposed to pick up the chopped tomato and deliver to the counter labelled by the star.
Meanwhile, the chef needs to take care of the pot that can probabilistically begin boiling, and the waiter needs to take care of a dine-in customer (labelled by the plate with fork and knife).
This introduces interesting tensions between delivering the food and taking care of the pot and the customer.
Please refer to Appendix \ref{appendix:Domain Description: Overcooked} for a detailed description of the environment.
 
For coordination, the chef makes a probabilistic commitment that it will place the chopped tomato on the plate.
Thus, the chef is the provider and the waiter is the recipient.
Crucially, the commitment decouples the agents' planning problems, allowing the agents to only model the MDP in their half of the grid.
We repeat the experiments in Section \ref{sec:Synthetic MDPs} that evaluate the breakpoints discretization and the greedy query over $50$ problem instances. 
We conjecture that, since the provider's transition function in Overcooked is more structured than in Section \ref{sec:Synthetic MDPs}, the breakpoints discretization is relatively smaller, leading to greater efficiency.
The results, presented in Figure \ref{fig:overcooked_EUS_runtime_vs_baselines} and Table \ref{table:overcooked discretiztion size} as the counterparts of Figure \ref{fig:EUS_runtime_vs_baselines} and Table \ref{table:discretization size}, confirm our conjecture.
Comparing Table \ref{table:overcooked discretiztion size} with Table \ref{table:discretization size}, we see that the breakpoints discretization in Overcooked is relatively smaller.
Therefore, it is unsurprising to see that the runtime in the Overcooked environment, as shown in Figure \ref{fig:overcooked_EUS_runtime_vs_baselines}(right), is relatively smaller than that in Figure \ref{fig:EUS_runtime_vs_baselines}(right).
The results that are the counterpart of Figure \ref{fig:EUS_runtime_vs_k} are presented in Appendix \ref{appendix:Greedy Query from the Breakpoints}, which confirm that the greedy query is again efficient and effective.
\begin{figure}
\centering
\begin{subfigure}{.23\textwidth}
  \centering
  \includegraphics[width=1\linewidth]{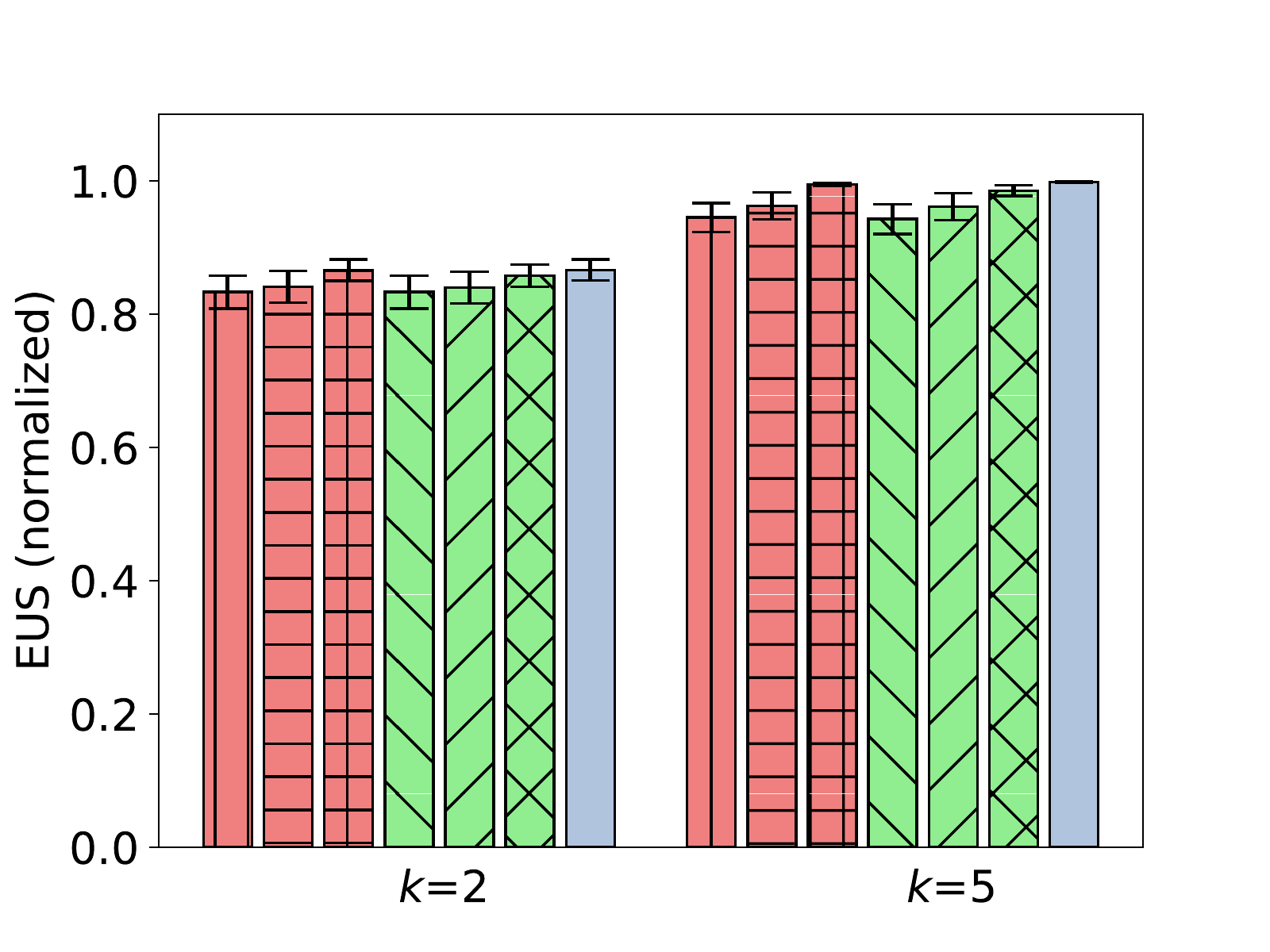}
%   \caption{EUS (normalized)}
%   \label{fig:overcooked_EUS_vs_baselines}
\end{subfigure}
\begin{subfigure}{.23\textwidth}
  \centering
  \includegraphics[width=1\linewidth]{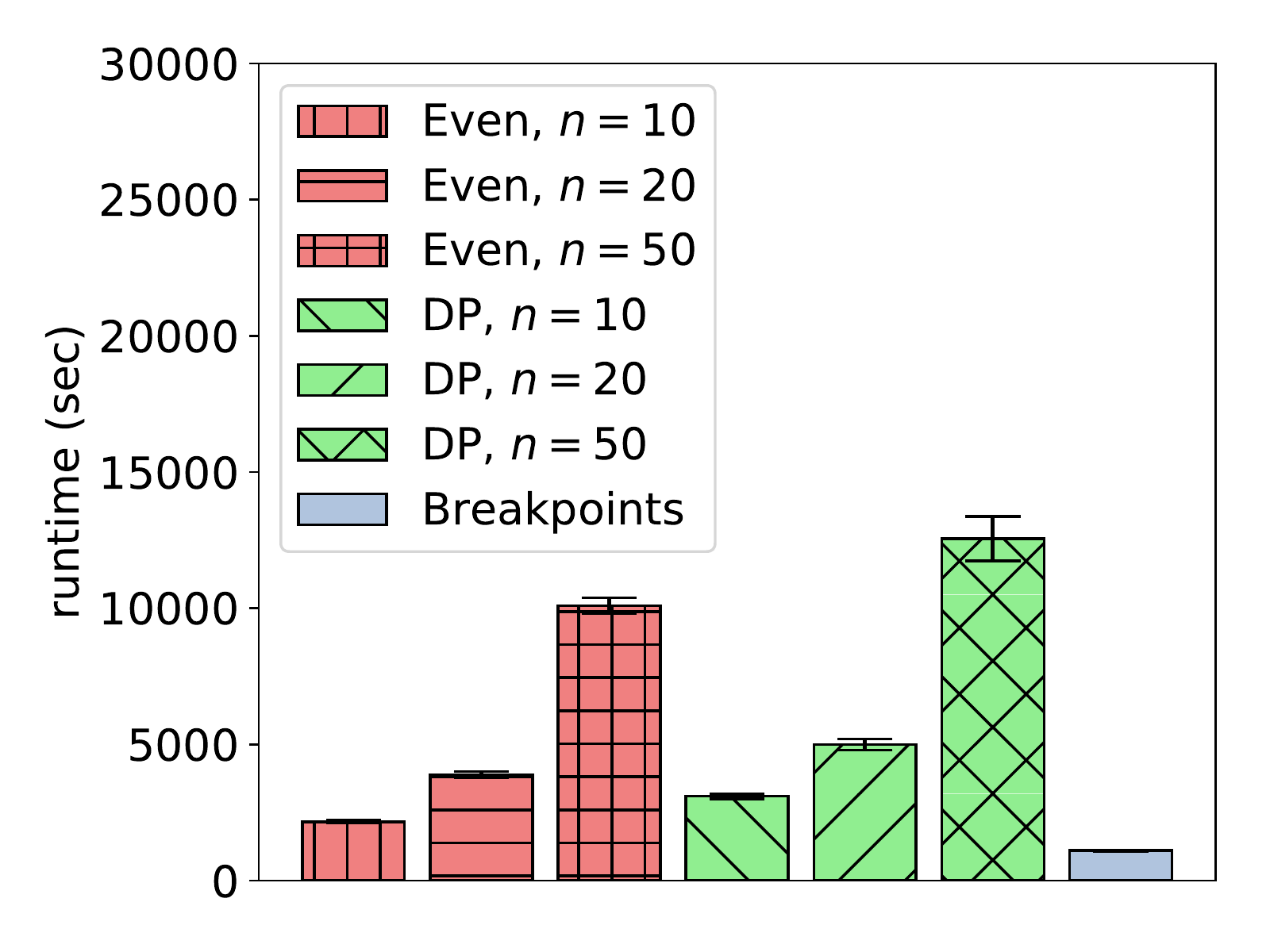}
%   \caption{runtime (sec.)}
%   \label{fig:overcooked_runtime_vs_baselines}
\end{subfigure}
\caption{Means and standard errors of the EUS (left) and runtime (right) 
% of the alternative discretizations 
in Overcooked.
% The runtime is for forming the discretization and evaluating the commitments in the discretization.
}
% \vspace{-2mm}
\label{fig:overcooked_EUS_runtime_vs_baselines}
\end{figure}

\begin{table}
\centering
\caption{Averaged discretization size per commitment time (mean and standard error) in Overcooked.}
\label{table:overcooked discretiztion size}

\begin{tabular}{c|ccc} 
\toprule
            & $n=10$        & $n=20$  & $n=50$            \\ 
\hline
Even        & $6.4 \pm 0.1$ & $11.8\pm0.3$ & $28.0\pm0.7$      \\
DP          & $5.2 \pm 0.2$ & $8.5\pm0.4$ & $16.4\pm0.9$       \\ 
\hline
Breakpoints & \multicolumn{3}{c}{$4.9\pm0.2$}  \\
\bottomrule
\end{tabular}
% \vspace{-1mm}
\end{table}

\paragraph{Quality of Coordination from Querying.}
The results thus far confirm that the agents are able to agree on a commitment from our querying process to achieve high expected joint commitment value (i.e. EUS).
By agreeing on $c$, the agents will execute joint policies $(\pi^{\rm p}(c), \pi^{\rm r}(c))$ derived from $c$, respectively.
Note that,  because $\pi^{\rm r}(c))$ corresponds to the recipient's approximation of the provider's true influence (Eq. \ref{eq:recipient commitment value}), the joint commitment value is not perfectly aligned with the value of the joint policies. 
A reasonable remaining question is: how effective is our commitment query approach in terms of maximizing joint policies' value, compared with the coordination approach originally examined in Overcooked, and with other approaches that make different tradeoffs with respect to observability, communication, and centralization?
To answer this question, we measure the following joint values:
1) a centralized planner view of the provider and the recipient as a single agent; this also corresponds to the multi-agent MDPs (MMDPs) model \cite{boutilier1996planning} where agents have global observability of the state and the reward when selecting actions.
2) decentralized MMDPs, which allows global observability but no centralization, so that the agents need to infer each other's intention individually; in this case, Wang el al. \shortcite{wang2020too} achieved values close to centralization in Overcooked;
3) centralized local observability, where a centralized planner yields joint policies that select actions based on local observability (i.e. half of the grid and private rewards); in particular, we consider policies derived from the optimal commitment (Section \ref{sec:Efficient Optimal Commitment Search}) found by the centralized planner;
4) decentralized local observability, which corresponds to our commitment query approach; we also consider the null commitment policy where the chef and the waiter only optimize the reward for the pot and the dine-in customer without delivering food. Note that this null commitment policy is a reasonable baseline if there is no communication allowed.

The results, presented in Table \ref{table:overcooked joint value of commitment query}, show that our commitment query approach uses modest communication to achieve joint values comparable to cases with stronger information infrastructures that assume centralization and/or global observability during planning and/or execution.
The results also confirm that careful selection of commitments for querying is crucial to induce effective coordination, as random commitments yield significantly lower joint values.

\begin{table}
\centering
\caption{Values of joint policies (mean and standard error in \%) in Overcooked, with MMDPs normalized to 100 and null commitment to 0.}
\label{table:overcooked joint value of commitment query}
% \begin{tabular}{cccc} 
% \toprule
% \begin{tabular}[c]{@{}c@{}}Random~\\Commitment\end{tabular} & \begin{tabular}[c]{@{}c@{}}Query\\ $k = 2$\end{tabular} & \begin{tabular}[c]{@{}c@{}}Query\\$k = 5$\end{tabular} & \begin{tabular}[c]{@{}c@{}}Optimal\\Commitment~ \end{tabular}  \\ 
% \midrule
% $14.5 \pm 1.4 \%$~                                             & $99.6 \pm 0.1\%$~                                      & $99.8 \pm 0.1\%$~                                      & $99.9 \pm 0.1\%$                                                 \\
% \bottomrule
% \end{tabular}

\begin{tabular}{l|c|c} 
\toprule
                                 & Centralized                                                & Decentralized                                                         \\ 
\hline
\multicolumn{1}{c|}{Global Obs.} & \begin{tabular}[c]{@{}c@{}}100\\(MMDPs) \end{tabular}      & \begin{tabular}[c]{@{}c@{}}near 100\\(Wang et al.) \end{tabular}      \\ 
\hline
\multicolumn{1}{c|}{Local Obs.}  & \begin{tabular}[c]{@{}c@{}}$99.9 \pm 0.1$\\(Optimal $c$) \end{tabular} & \begin{tabular}[c]{@{}c@{}}$99.6 \pm 0.1$, $99.8 \pm 0.1$\\(Query $k = 2$, $5$) \end{tabular}  \\ 
\cline{2-3}
                                 & \multicolumn{2}{c}{Null $c$: $0$;~~~~Random $c$: $14.5 \pm 1.4$ }                                                                                               \\
\bottomrule
\end{tabular}

\end{table}

\section{Discussion}
\label{sec:conclusions}
Built on provable foundations and  evaluated in two separate domains, our approach proves highly appropriate for settings where cooperative agents coordinate their plans through commitments in a decentralized manner, and could provide a good performance/cost tradeoff even compared to coordination that is not restricted to being commitment-based. 

For future directions, if the agents can afford the time and bandwidth, querying need not be limited to a single round, which then raises questions about how agents should consider future rounds when deciding on what to ask in the current round. 
The querying can also be extended to the setting where the query poser is uncertain about both the responder's and its own environments.  
% As pointed out earlier, the provider and recipient could possibly switch roles of who poses the query and who responds. 
As dependencies between agents get richer (with chains and even cycles of commitments), continuing to identify and exploit structure in intertwined value functions will be critical to scaling up for efficient multi-round querying of connected commitments.

\noindent{\bf Acknowledgments}
We thank the anonymous reviewers for their valuable feedback.
This work was supported in part by the Air Force Office of Scientific Research under grant FA9550-15-1-0039. Opinions, findings, conclusions, or recommendations expressed here are those of the authors and do not necessarily reflect the views of the sponsor.

\bibliography{ref}

\newpage\clearpage
\appendix
\section*{Appendix}
\vspace{3mm}
\section{Proofs}
\label{appendix:Proofs}
\subsection{Proof of Theorem \ref{theorem:provider commitment value}}
\paragraph{Proof of monotonicity}
By the commitment semantics of Eq.~\eqref{eq:commitment semantics}, $\Pi^{\rm p}(c) = \Pi^{\rm p}(T, p)$ is monotonically non-increasing in $p$ for any fixed $T$, i.e.  $\Pi^{\rm p}(T,p') \subseteq \Pi^{\rm p}(T,p)$ for any $p'>p$. Therefore, $v^{\rm p}(T, p)$ is monotonically non-increasing in $p$.

\paragraph{Proof of concavity}
Consider the linear program (LP), patterned on the literature\cite{altman1999constrained,Witwicki2007CommitmentdrivenDJ},
that solves the provider's planning problem in Eq.~\eqref{eq:provider commitment value}:
\begin{subequations}
\begin{align} 
\max_{x}  ~&\sum_{s^{\rm p},a^{\rm p}}x(s^{\rm p},a^{\rm p}) R^{\rm p}(s^{\rm p},a^{\rm p}) \label{program:Obj}\\
\mathrm{s.t.}~~&\forall s^{\rm p}, a^{\rm p} \quad x(s^{\rm p},a^{\rm p})\geq 0;  \label{program:x geq zero}\\
&\forall s^{p\prime} \quad \sum_{a^{p\prime}}x(s^{p\prime},a^{p\prime}) \label{program:dynamics}\\
&~~~~~=\sum_{s^{\rm p},a^{\rm p}}x(s^{\rm p},a^{\rm p})P^{\rm p}(s^{p\prime}|s^{\rm p},a^{\rm p})+\delta(s^{p\prime},s^{\rm p}_0); \nonumber\\
&\sum_{s^{\rm p}\in s^+_T} \sum_{a^{\rm p}}x(s^{\rm p},a^{\rm p}) \geq p \label{program:commitment}
\end{align}
\end{subequations}
where $\delta(s^{p\prime},s^{\rm p}_0)$ is the Kronecker delta that returns 1 when $s^{p\prime}=s^{\rm p}_0$ and 0 otherwise, and 
$
    s^+_T =\{ s^{\rm p}: s^{\rm p}\in S^{\rm p}_T, u^+ \in s^{\rm p} \}
$
is the set of the provider's states at commitment time $T$ in which $u=u^+$.
If $x$ satisfies constraints \eqref{program:x geq zero} and \eqref{program:dynamics},  then it is the occupancy measure of policy $\pi^{\rm p}$,
\begin{align*}
\pi^{\rm p}(a^{\rm p}|s^{\rm p})=\frac{x(s^{\rm p},a^{\rm p})}{\sum_{a^{p\prime}}x(s^{\rm p},a^{p\prime})},
\end{align*}
% inserted
where $x(s^{\rm p}, a^{\rm p})$ is the expected number of times action $a^{\rm p}$ is taken in state $s^{\rm p}$ by following policy $\pi^{\rm p}$.
Constraint \eqref{program:commitment} expresses the commitment semantics of Eq.~\eqref{eq:commitment semantics}.
The expected cumulative reward is expressed in the objective function \eqref{program:Obj}.
Therefore, $v^{\rm p} (c)$ is the optimal value of this linear program.

For a fixed commitment time $T$ and any two commitment probabilities $p$ and $p'$, let $x^*_{p}, x^*_{p'}$ be the optimal solutions to the LP, respectively.
For any $\eta\in[0,1]$, let $p_\eta = \eta p' + (1-\eta) p$. Consider $x_\eta$ that is the $\eta$-interpolation of $x^*_{p}, x^*_{p'}$,
\begin{align*}
    x_\eta(s^{\rm p},a^{\rm p}) =  \eta x^*_{p'}(s^{\rm p},a^{\rm p}) + (1-\eta) x^*_{p}(s^{\rm p},a^{\rm p}).
\end{align*}
Note that $x_\eta$ satisfies constraints \eqref{program:x geq zero} and \eqref{program:dynamics}, and so it is the occupancy measure of policy $\pi^{\rm p}_\eta$ defined as 
\begin{align*}
\pi^{\rm p}_\eta(a^{\rm p}|s^{\rm p})=\frac{x_\eta(s^{\rm p},a^{\rm p})}{\sum_{a^{p}}x_\eta(s,a^{p})}.
\end{align*}
Since the occupancy measure of $\pi^{\rm p}_\eta$ is the $\eta$-interpolation of $x^*_{p}$ and $x^*_{p'}$,
it is easy to verify that $\pi^{\rm p}_\eta$ is feasible for commitment probability $p_\eta$.
Therefore, the concavity holds because
\begin{align*}
     &v^{\rm p}(T, p_\eta) \\
    \geq& V^{\pi^{\rm p}_\eta}_{M^{\rm p}}(s^{\rm p}_0)
    = \sum_{s^{\rm p},a^{\rm p}}x_\eta(s^{\rm p},a^{\rm p}) R^{\rm p}(s^{\rm p},a^{\rm p}) \\
    =&  \sum_{s^{\rm p},a^{\rm p}}  \left( \eta x^*_{p'}(s^{\rm p},a^{\rm p}) + (1-\eta) x^*_{p_0}(s^{\rm p},a^{\rm p}) \right)  R^{\rm p}(s^{\rm p},a^{\rm p}) \\
    =& \eta v^{\rm p}(T, p') + (1-\eta) v^{\rm p}(T, p_0).
\end{align*}

\paragraph{Proof of piecewise linearity}
We first convert the original linear program into its standard form:
\begin{subequations}
\begin{align} 
\max_{\tilde{x}} ~~~ & r^T\tilde{x} \nonumber\\ %\label{program:Standard Form Obj}\\
\mathrm{s.t.}~  
& A \tilde{x} = b; \nonumber\\ %\label{program:Standard Form Equality}\\
&~\tilde{x}\geq 0;  \nonumber %\label{program:Standard Form x geq zero}
\end{align}
\end{subequations}
To convert constraint \eqref{program:commitment} into an equality constraint, we introduce a slack variable $\xi \geq 0$:
\begin{align*}
    \sum_{s^{\rm p} \in S^{p+}_T} \sum_{a^{\rm p}} x(s^{\rm p},a^{\rm p}) - \xi = p.
\end{align*}
The slack variable is a decision variable in the standard form, $\tilde{x} = [x ~|~ \xi] \in \mathbb{R}^{|S^{\rm p}||A^{\rm p}| + 1}$.
The standard form eliminates redundant constraints so that $A\in\mathbb{R}^{m \times (|S||A| + 1)}$ is full row rank ($rank(A) = m$). Note that the elimination produces $b\in\mathbb{R}^m$ whose elements are linear in $p$.

Pick a set of indices $B$ corresponding to $m$ columns of the matrix $A$. We can think of $A$ as the concatenation of two matrices $A_B$ and $A_N$ where $A_B$ is the $m \times m$ matrix of these $m$ linearly independent columns, and $A_N$ contains the other columns. Correspondingly, $\tilde{x}$ is decomposed into $\tilde{x}_B$ and $\tilde{x}_N$.
Then, $\tilde{x} = [\tilde{x}_B ~|~ \tilde{x}_N]$ is basic feasible if $x_N=0$, $A_B$ is invertible, and $x_B = A_B^{-1} b \geq 0$.

It is known that the optimal solution can be found in the basic feasible solutions,
\begin{align*}
    v^{\rm p}(T, p) 
    &= \max_{B: \tilde{x} \text{ is basic feasible} }  r^T\tilde{x} \\
    &= \max_{B: \tilde{x} \text{ is basic feasible} }  r^T_B \tilde{x}_B \\
    &= \max_{B: \tilde{x}  \text{ is basic feasible} }  r^T_B A_B^{-1} b.
\end{align*}
Since $b$ is in linear in $p$, $v^{\rm p}(T, p)$ is the maximum of a set of linear functions in $p$, and therefore it is piecewise linear.

\subsection{Proof of Theorem \ref{theorem:recipient commitment value}}

\paragraph{Proof of monotonicity.}
We fix the commitment time $T$.
For any recipient policy $\pi^r$, let $v^{\pi^r}_{T, 1}$ be the initial state value of $\pi^r$ when $u$ is enabled from $u^-$ to $u^+$ with probability 1 at $T$, and let $v^{\pi^r}_{T, 0}$ be the initial state value of $\pi^r$ when $u$ never flips to $u^+$.
It is useful to notice that 
\begin{align} \label{eq:value decomposition}
    V^{\pi^r}_{\widehat{M}^r(c)} (s^r_0) = p v^{\pi^r}_{T, 1} + (1-p) v^{\pi^r}_{T, 0}
\end{align}
In words, the initial state value can be expressed as the weighted sum of the two scenarios, with the weight determined by the commitment probability.
Consider the optimal policy $\pi^*_{\widehat{M}^r(c)}$ for $\widehat{M}^r(c)$.
It is guaranteed that $v^{\pi^*_{\widehat{M}^r(c)}}_{T, 1} \geq  v^{\pi^*_{\widehat{M}^r(c)}}_{T, 0}$ because, intuitively, $u^+$ is more desirable than $u^-$ to the recipient. We will formally prove this later.
Now consider $p' > p$ and let $c'= (t, p')$:
\begin{align*}
     v^r(T, p)
    &= p v^{\pi^*_{\widehat{M}^r(c)}}_{T, 1} + (1-p) v^{\pi^*_{\widehat{M}^r(c)}}_{T, 0} \\
    &\leq p' v^{\pi^*_{\widehat{M}^r(c)}}_{T, 1} + (1-p') v^{\pi^*_{\widehat{M}^r(c)}}_{T, 0}
    \leq v^r(T, p').
\end{align*}

Now, we finish the proof by formally showing $v^{\pi^*_{\widehat{M}^r(c)}}_{T, 1} \geq  v^{\pi^*_{\widehat{M}^r(c)}}_{T, 0}$.
To this end, it is useful to first give Lemma \ref{lemma:M+-} that directly follows from Assumption \ref{assumption:u}, stating that the value when $u$ is always set to $u^-$ is no more than the value of any arbitrary $M^r$.
\begin{lemma}\label{lemma:M+-}
Under Assumption \ref{assumption:u}, 
for any $M^r$ with arbitrary $P^r_u$,
we have $V^*_{M^{r-}}(s^{r}_0) \leq V^*_{M^r}(s^r_0)$.
\end{lemma}

\begin{proof}
Let's first consider the case in which $P^r_u$ flips $u$ only at a single time step $T$.
We show $V^*_{M^{r-}}(s^{r}_0) \leq V^*_{M^r}(s^r_0)$ by constructing a policy in $M^r$ for which the value is at least $V^*_{M^{r-}}(s^{r}_0)$ by mimicking $\pi^*_{M^{r-}}$.
We can construct a policy $\pi_{M^r}$ that chooses the same actions as $\pi^*_{M^{r-}}$ up until time step $T$.
If $u$ is not toggled at $T$, then we keep choosing the same actions as $\pi^*_{M^{r-}}$ throughout the episode; 
otherwise, after $T$ we chooses actions that are optimal for $u^+$.
By Assumption \ref{assumption:u}, this policy yields a value that is at least $V^*_{M^{r-}}(s^{r}_0)$.

For the case in which $P^r_u$ flips $u$ with positive probability at $K>1$ time steps, we can decompose the value function for $P^r_u$ as the weighted average of $K$ value functions, each of which corresponds to the scenario where $u$ only flips at a single time step, and the weights of the average are the flipping probabilities of $P^r_u$ at these $K$ time steps.
\end{proof}
Now we can show $v^{\pi^*_{\widehat{M}^r(c)}}_{T, 1} \geq  v^{\pi^*_{\widehat{M}^r(c)}}_{T, 0}$ because, otherwise, we have
\begin{align*}
    v^{r} (T, p) 
    &= p v^{\pi^*_{\widehat{M}^r(c)}}_{T, 1} + (1-p) v^{\pi^*_{\widehat{M}^r(c)}}_{T, 0} \\
    &< p  v^{\pi^*_{\widehat{M}^r(c)}}_{T, 0} + (1-p) v^{\pi^*_{\widehat{M}^r(c)}}_{T, 0} \\
    & =  v^{\pi^*_{\widehat{M}^r(c)}}_{T, 0}
    \leq V^*_{M^-} (s^{r-}_0)
\end{align*}
where $v^{r} (T, p)<V^*_{M^-} (s^{r-}_0)$ contradicts Lemma \ref{lemma:M+-}.

\paragraph{Proof of convexity and piecewise linearity.}
Let $\Pi^r_D$ be the set of all the recipient's deterministic policies. It is well known~\cite{puterman2014markov}
that the optimal value can be attained by a deterministic policy,
\begin{align*}
    v^r(T, p)
    = \max_{\pi^r \in \Pi^r_D }  V^{\pi^r}_{\widehat{M}^r(c)}(s^r_0) 
    = \max_{\pi^r \in \Pi^r_D }  p v^{\pi^r}_{T, 1} + (1-p) v^{\pi^r}_{T, 0}
\end{align*}
which indicates that $v^r(T, p)$ is the maximum of a finite number of value functions that are linear in $p$.
Therefore, $v^r(T, p)$ is convex and piecewise linear in $p$.

\subsection{Proof of Theorem \ref{theorem: breakpoints}}

As an immediate consequence of Theorems \ref{theorem:provider commitment value} and \ref{theorem:recipient commitment value}, the joint commitment value is piecewise linear in the probability, and any local maximum for a fixed commitment time $T$ can be attained by a breakpoint probability. Therefore, restricting to those breakpoint commitments incurs no loss of optimality.

\subsection{Proof of Theorem \ref{theorem: submodularity}}
Since the recipient always chooses the one that maximizes the joint value over all commitments in the query, this reduces to the scenario referred to as the noiseless response model in prior work on EUS maximization~\cite{viappiani2010optimal}. 
\cite{viappiani2010optimal} proves the submodularity under the noiseless response model, which also proves Theorem \ref{theorem: submodularity}.

% \subsection{Proof of Lemma \ref{lemma: breakpoints} }
% Placeholder.

\subsection{Proof of Theorem \ref{theorem: optimal discretization} }
We first give Lemma \ref{lemma: optimal discretization} that says any discretization that contains the linearity breakpoints is no worse than any other discretization.

\begin{lemma}
\label{lemma: optimal discretization}
Let $\mathcal{C}$ be defined in the same manner as in Theorem \ref{theorem: optimal discretization}.
Consider any finite set of commitments $\overline{\mathcal{C}}$ that contains $\mathcal{C}$, i.e.  $\overline{\mathcal{C}}\supseteq\mathcal{C}$. 
For any query size $k$ and any uncertainty $\mu$, 
\begin{align} \label{eq: lemma of optimal discretization}
  \max_{Q\subseteq \overline{\mathcal{C}}, |Q|=k} EUS(\mathcal{Q};\mu) 
  =
  \max_{Q\subseteq \mathcal{C}, |Q|=k} EUS(\mathcal{Q};\mu).
\end{align}
\end{lemma}
\begin{proof}
Because $\overline{\mathcal{C}}\supseteq\mathcal{C}$, it is obvious that ``$\geq$'' holds for Eq.~\eqref{eq: lemma of optimal discretization}. We next show ``$\leq$''.

Given a commitment query $\mathcal{Q} = \{c_1,...,c_k\}$, define $T(\mathcal{Q})$ as a commitment query where each commitment is the optimal commitment with respect to the posterior given a response for $Q$, i.e. 
\begin{align*}
    T(\mathcal{Q}) =\{c^*(\mu~|~\mathcal{Q} \rightsquigarrow c_1),..., c^*(\mu~|~\mathcal{Q} \rightsquigarrow c_k)\}.
\end{align*}
Previous work~\cite{viappiani2010optimal} shows that $EUS(T (\mathcal{Q}); \mu) \geq EUS(\mathcal{Q}; \mu)$.
Due to Lemma \ref{lemma: breakpoints}, we now have $c^*(\mu) \in \mathcal{C}$ for any uncertainty $\mu$.
Thus, given an EUS maximizer $\mathcal{Q}^*$ for $\overline{\mathcal{C}}$, $T(\mathcal{Q}^*)$ is a subset of $\mathcal{C}$ with a EUS that is no smaller, which shows ``$\leq$'' holds for Eq.~\eqref{eq: lemma of optimal discretization}.
This concludes the proof.
\end{proof}

We are ready to prove Theorem \ref{theorem: optimal discretization}.
Consider the even discretization of $[0, 1]$, $\mathcal{P}_n = \{p_0,$ $p_1, ..., p_n\}$ where $p_i = \frac{i}{n}$.
Because $v^{p+r}$ is bounded and piecewise linear in the commitment probability, for any $\epsilon>0$, there exists a large enough discretization resolution $n$, such that for any size $k$ query $\mathcal{Q}\subseteq \mathcal{T}\times [0,1]$, there is a size $k$ query $\widehat{\mathcal{Q}} \in \mathcal{T}\times\mathcal{P}_n$ that $|EUS(\mathcal{Q};\mu)-EUS(\widehat{\mathcal{Q}};\mu)|\leq\epsilon$.
Therefore, we have
\begin{align*}
     EUS(\mathcal{Q};\mu) - \epsilon 
  \leq&
  \max_{\widehat{\mathcal{Q}}\subseteq \mathcal{T}\times \mathcal{P}_n , |\widehat{\mathcal{Q}}|=k} EUS(\widehat{\mathcal{Q}};\mu) \\
  \leq&
  \max_{\mathcal{Q}\subseteq (\mathcal{C} \cup \mathcal{T}\times \mathcal{P}_n), |\mathcal{Q}|=k} EUS(\mathcal{Q};\mu)\\
  =&
  \max_{\mathcal{Q}\subseteq \mathcal{C}, |\mathcal{Q}|=k} EUS(\mathcal{Q};\mu)
\end{align*}
for any query $\mathcal{Q}\subseteq \mathcal{T}\times [0,1]$ with $|\mathcal{Q}|=k$, where the equality is a direct result from Lemma \ref{lemma: optimal discretization}. 
This concludes the proof.

\section{Domain Description: Synthetic MDPs}
\label{appendix:Domain Description: Synthetic MDPs}
\textbf{The provider's environment} is a randomly-generated MDP, from a distribution designed such that, in expectation, the provider's reward when enabling the precondition is smaller than when not enabling it. This introduces tension in the provider between enabling the precondition to help the recipient, versus increasing its own reward.

We now describe the provider's MDP-generating distribution.
The MDP has 10 states the provider can be in at any time step, one out of which is an absorbing state denoted as $s^+$, and where the initial state is chosen from the non-absorbing states.
Feature $u$ takes the value of $u^+$ only in the absorbing state, i.e. $u^+ \in s^{\rm p}$ if and only if $s^{\rm p} = s^+$.
There are 3 actions.
For each state-action pair $(s^{\rm p},a^{\rm p})$ where $s^{\rm p}\neq s^+$, the transition function $P^{\rm p}(\cdot|s^{\rm p},a^{\rm p})$ is determined independently by filling the 10 entries with values uniformly drawn from $[0, 1]$, and normalizing $P^{\rm p}(\cdot|s^{\rm p},a^{\rm p})$.
The reward $R^{\rm p}(s^{\rm p}, a^{\rm p})$ for a non-absorbing state $s^{\rm p}\neq s^+$ is sampled uniformly and independently from $[0, 1]$, and for the absorbing state $s^{\rm p}= s^+$ is zero, meaning the provider prefers to avoid the absorbing state, but that state is the only one that satisfies the commitment.

\textbf{The recipient's environment}, inspired by the random walk domains used in the planning literature~\cite{fern2004learning,nakhost2009monte}, is a one-dimensional space with $10$ locations represented as integers $\{0,1,...,9\}$, as illustrated in Figure \ref{fig:1DWalk}.
In locations $1-8$, the recipient can move right, left, or stay still.
Once the recipient reaches either end (location $0$ or $9$), it stays there.
There is a gate between locations $0$ and $1$ for which $u=u^+$ denotes the state of open and $u=u^-$ closed.
Initially, the gate is closed and the recipient starts at an initial location $L_0$.
A negative reward of $-10$ is incurred by bumping into the closed gate.
For each time step the recipient is at neither end, it gets a reward of $-1$.
If it reaches the left end (i.e. location 0), it gets a one-time reward of $r_{0}\geq0$.
The recipient gets a reward of 0 if it reaches the right end.
In a specific instantiation of the recipient's MDP, $L_0$ and $r_{0}$ are fixed, and they are randomly chosen to create various MDPs for the recipient. $L_0$ is randomly chosen from locations $1-8$ and $r_{0}$ from interval $[0,10]$.
\begin{figure}[ht]
\begin{center}
    \includegraphics[width=.7\linewidth]{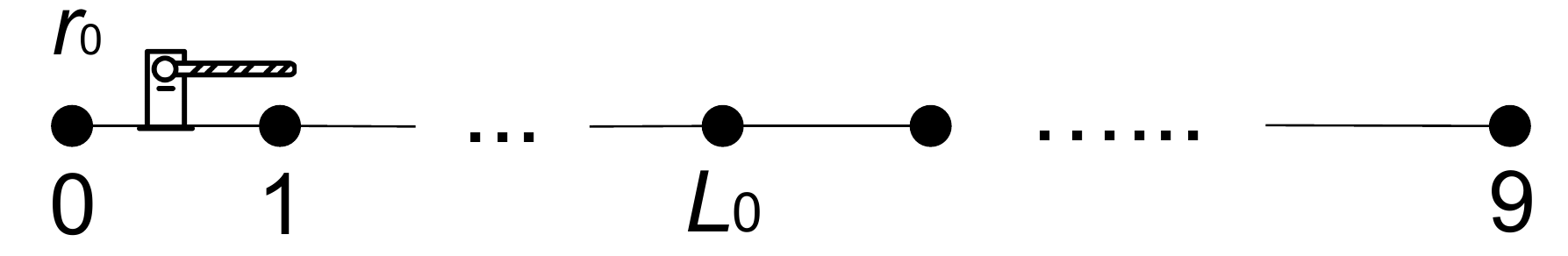}
\end{center}
\caption{1D Walk as the recipient's environment in Section \ref{sec:Synthetic MDPs}.}
\label{fig:1DWalk}
\end{figure}

To generate a random coordination problem, we sample an MDP for the provider, and $N$ candidate MDPs for the recipient, setting the provider's prior uncertainty $\mu$ over the recipient's MDP to be the uniform distribution over the $N$ candidates.
The horizon for both agents is set to be $H=H^{\rm p}=H^{\rm r}=20$.
Since the left end has higher rewards than the right end, if the recipient's start position is close enough to the left end and the provider commits to opening the gate early enough with high enough probability, the recipient should utilize the commitment by checking if the gate is open by the commitment time, and pass through it if so; 
otherwise, the recipient should simply ignore the commitment and move to the right end.
The distribution for generating the recipient's MDPs is designed to include diverse preferences regarding the commitments, such that the provider's query should be carefully formulated to elicit the recipient's preference.

\section{Domain Description: Overcooked}
\label{appendix:Domain Description: Overcooked}
The domain, Overcooked, was inspired by the video game of the same name and introduced by \cite{wang2020too}.
It is a gridworld domain that requires the agents to cooperate in an environment that mimics a restaurant.
We use an environment of Overcooked as is in \cite{wang2020too} with two high-level modifications:
1) instead of having global observability, each agent observes its local environment,
and 2) we introduce probabilistic effects into the transition function.
We make these modifications to induce for the domain a rich space of meaningful commitments, over which the agents should carefully negotiate for the optimal cooperative behavior.

% \begin{figure}[th]
% \centering
%   \includegraphics[width=.57\linewidth]{figs_overcooked/overcooked.pdf}
% \caption{Overcooked.}
% \label{fig:overcooked}
% \end{figure}

Figure \ref{fig:overcooked} illustrates this Overcooked environment, and we here describe it in detail.
Two agents, the chef and the waiter, together occupy a 7x7 grid with counters being the boundaries. Counters divide the grid into halves with the chef on the left and the waiter on the right.
The chef is supposed to pick up the tomato, chop it, and place it on the plate. Afterwards, the waiter is supposed to pick up the chopped tomato and deliver to the counter labelled by the star.
Meanwhile, the chef needs to take care of the pot, and the waiter needs to take care of a dine-in customer (labelled by the plate with fork and knife).
Specifically, each agent has nine actions: \{N,E,S,W\}-move, \{N,E,S,W\}-interact, and do-nothing. 
The \{N,E,S,W\}-move actions change the agents' location in cardinal directions.
The \{N,E,S,W\}-interact actions change the status of the object in the corresponding cardinal directions:
the chef picks up the (unchopped) tomato by interacting with it; 
after picking up the tomato, the chef chops it (and keeps carrying it) by interacting with the knife;
after chopping the tomato, the chef places it on the plate by interacting with the plate that is initially empty;
the waiter picks up the tomato on the plate by interacting with the plate;
after picking up the tomato, the waiter delivers it by interacting with the counter labelled by the star;
being initially unboiled, at each time step the pot can turn boiling with probability $p_{\rm boiling}$, and when it is boiling the chef can turn the heat off by interacting with it;
except for the aforementioned cases, the interact actions has no effect (equivalent to do-nothing).
The chef get a reward -1 for every time step the pot is boiling.
The waiter gets a positive reward $r_{\rm delivery}$ upon the delivery.
At every time step, the waiter also gets a negative reward $d r_{\rm distance}$, where $d$ is the Manhattan distance between the waiter and the dine-in customer, and $r_{\rm distance}$ is a negative number, which encourages the waiter to stay close to the dine-in customer.

To facilitate coordination between the two agents, we consider commitments concerning the tomato, where the chef makes a commitment that it will place the chopped tomato on the plate by some time step with at least a certain probability.
Thus, the chef is the provider and the waiter is the recipient.
Crucially, the commitment decouples the agents' planning problems, allowing the chef to only model in its MDP the left half of the grid and the waiter to only model the right half.
Similarly to Section \ref{sec:empirical}, we generate a random coordination problem by sampling an MDP for the chef, and 10 candidate MDPs for the waiter.
For the chef's MDP, the initial locations of the chef, the tomato, the pot, and the knife are random and different, while the plate is always on the counter shown in Figure \ref{fig:overcooked}.
For the waiter's MDP, the initial locations of the waiter, the delivery counter, and the dine-in customer are random and different.
The probability $p_{\rm boiling}$ is uniformly sampled from $[0, 0.1]$, $r_{\rm distance}$ from $[0, 0.1]$, and $r_{\rm delivery}$ from $[5,15]$.
Knowing $p_{\rm boiling}$ but not $r_{\rm distance}$ and $r_{\rm delivery}$, the chef should carefully formulate the commitment query to elicit the commitment that balances the tradeoff between delivering the food and taking care of the pot and the dine-in customer.

\section{Supplementary Results in Synthetic MDPs}
All experiments were run on Intel Xeon E5-2630 v4 (2.20GHz) CPUs, with mip and pulp Python packages as (MI)LP solvers.
\label{appendix:Supplementary Results in Synthetic MDPs}

\subsection{Diverse Priors and Multi-round Querying}
\label{appendix:Diverse Priors and Multi-round Querying}
Figure \ref{fig:EUS_runtime_vs_k} has demonstrated the effectiveness of the greedy query for a particular type of the provider's prior $\mu$, which is the uniform distribution over the recipient's $N=10$ candidate MDP.
Here, we further show that the greedy query's effectiveness is robust to diverse prior types.
Besides the uniform prior, we consider two other prior types.
For the random prior, the probability for each candidate recipient's MDP is proportional to a number that is randomly sampled from interval $[0,1]$.
For the Gaussian prior, the probability for each candidate recipient's MDP is proportional to the standard Gaussian distribution's probability density function evaluated at a number randomly sampled from the three-sigma interval $[-3,3]$.
Figure \ref{fig:prior_type} shows the EUS, normalized in the same manner as Figure \ref{fig:EUS_runtime_vs_k}, of the greedy query for the three prior types, with the number of candidate recipient's MDPs $N=10$, and $50$.
For comparison, Figure \ref{fig:prior_type} shows, for query size $k=1,2,3$, the EUS of the optimal query and the greedy query's theoretical lower bound ($1-(\frac{k-1}{k})^k$ of the EUS of the optimal query of size $k$).

\begin{figure}[ht]
\centering

\begin{subfigure}[b]{.23\textwidth}
  \centering
  \includegraphics[width=\linewidth]{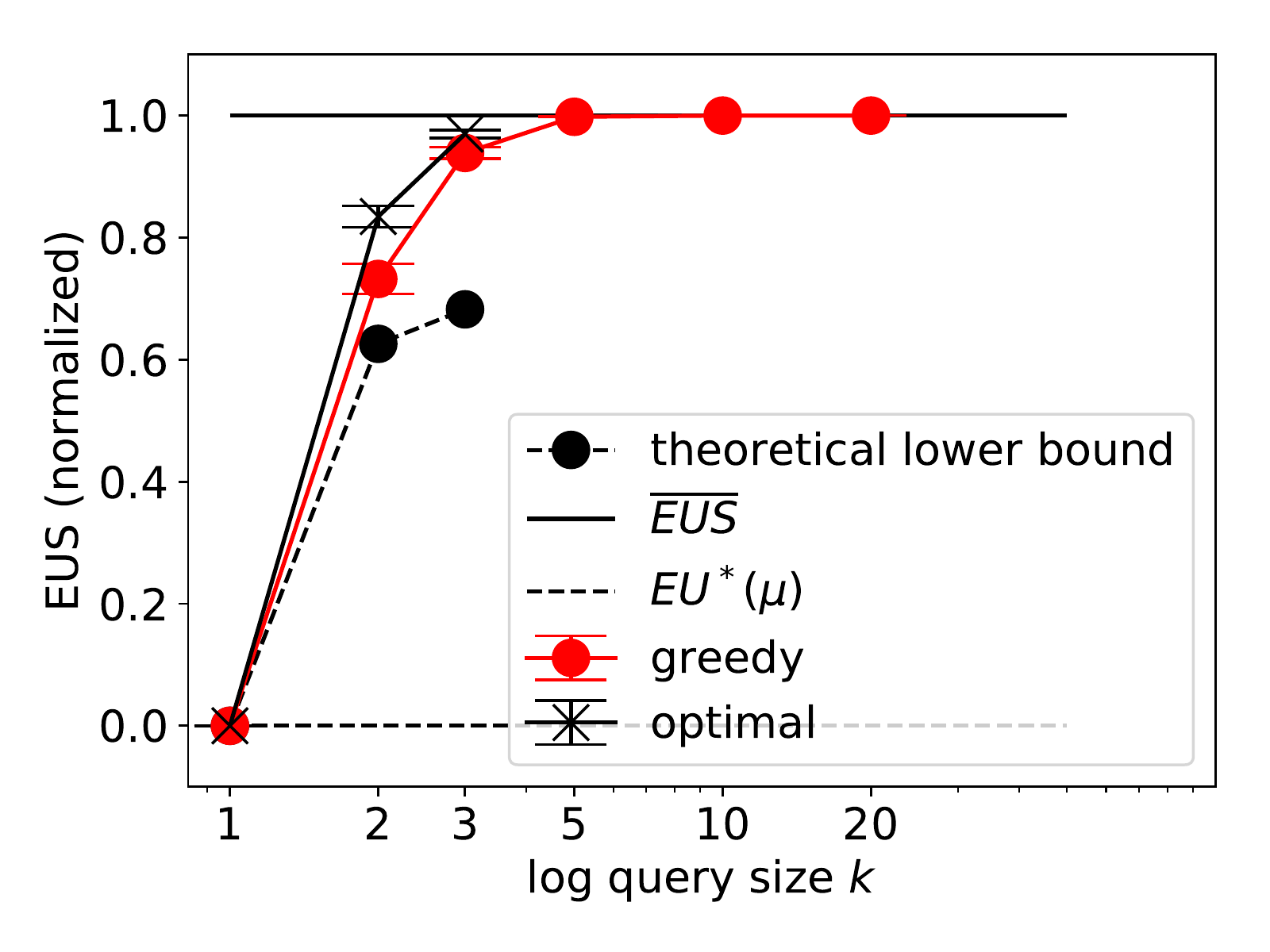}
  \caption{Uniform Prior, $N=10$}
  \label{fig:uniform_prior_N10}
\end{subfigure}
\begin{subfigure}[b]{.23\textwidth}
  \centering
  \includegraphics[width=\linewidth]{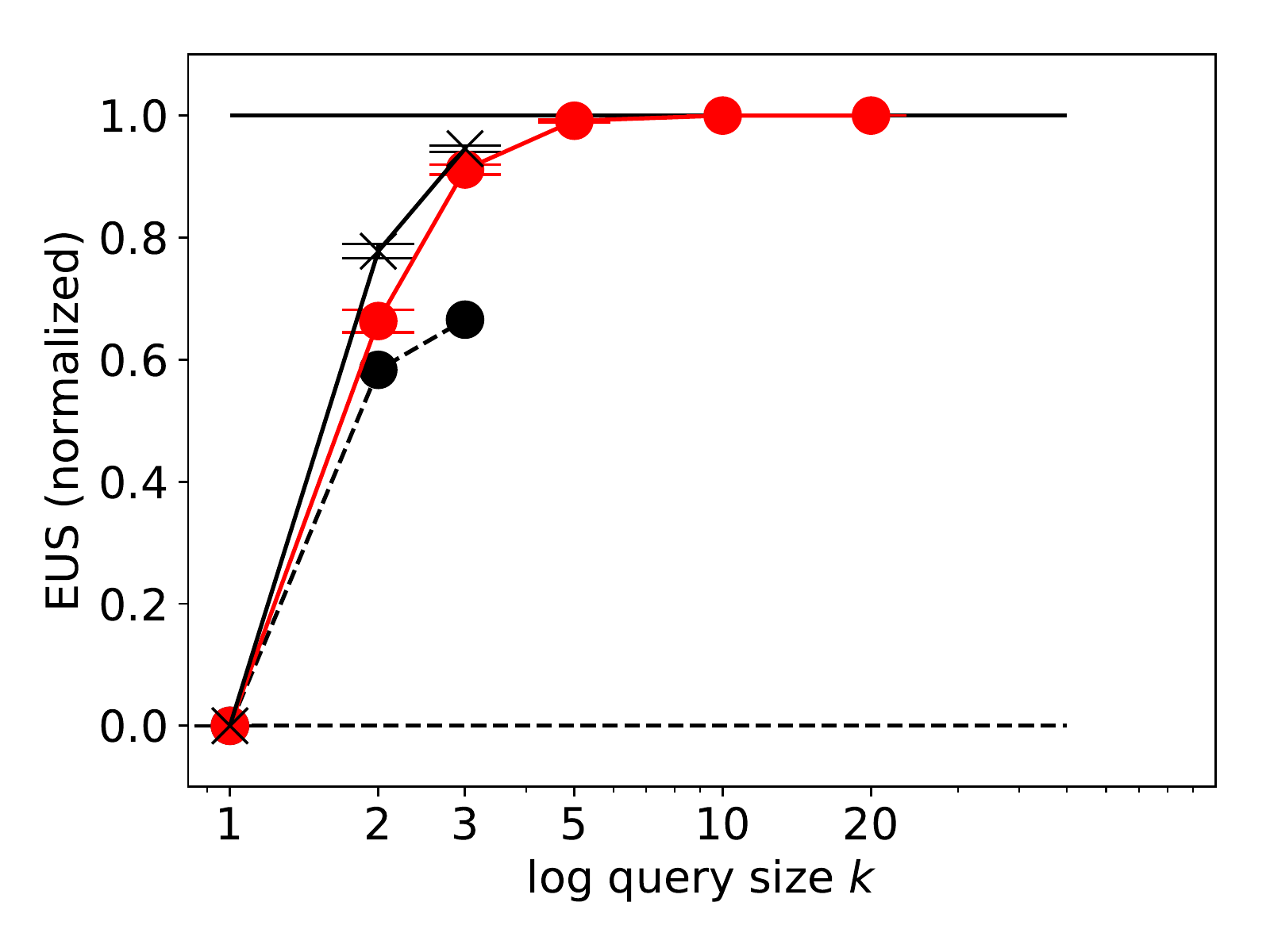}
  \caption{Uniform Prior, $N=50$}
  \label{fig:uniform_prior_N50}
\end{subfigure}

\begin{subfigure}[b]{.23\textwidth}
  \centering
  \includegraphics[width=\linewidth]{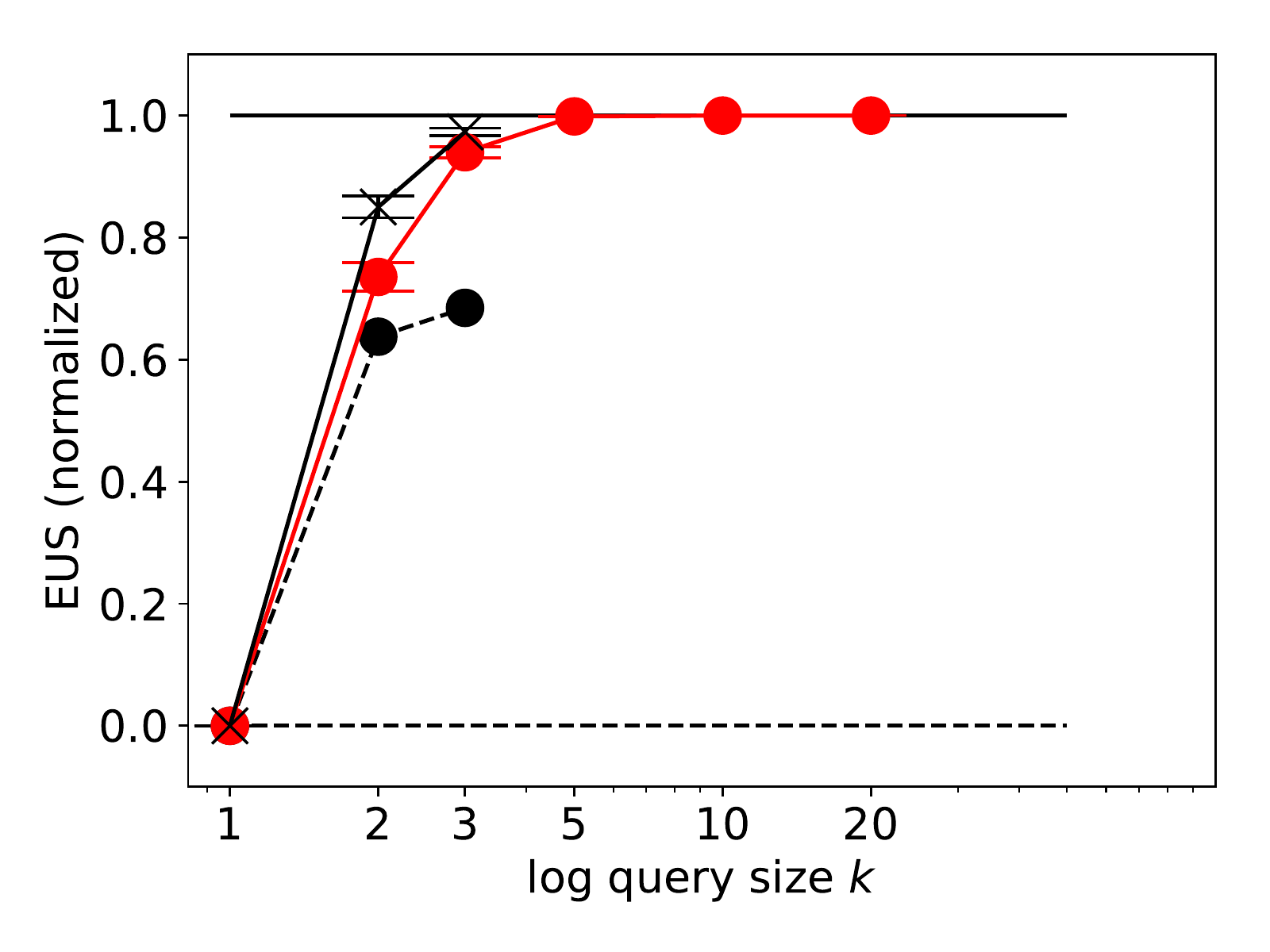}
  \caption{Random Prior, $N=10$}
  \label{fig:random_prior_N10}
\end{subfigure}
\begin{subfigure}[b]{.23\textwidth}
  \centering
  \includegraphics[width=\linewidth]{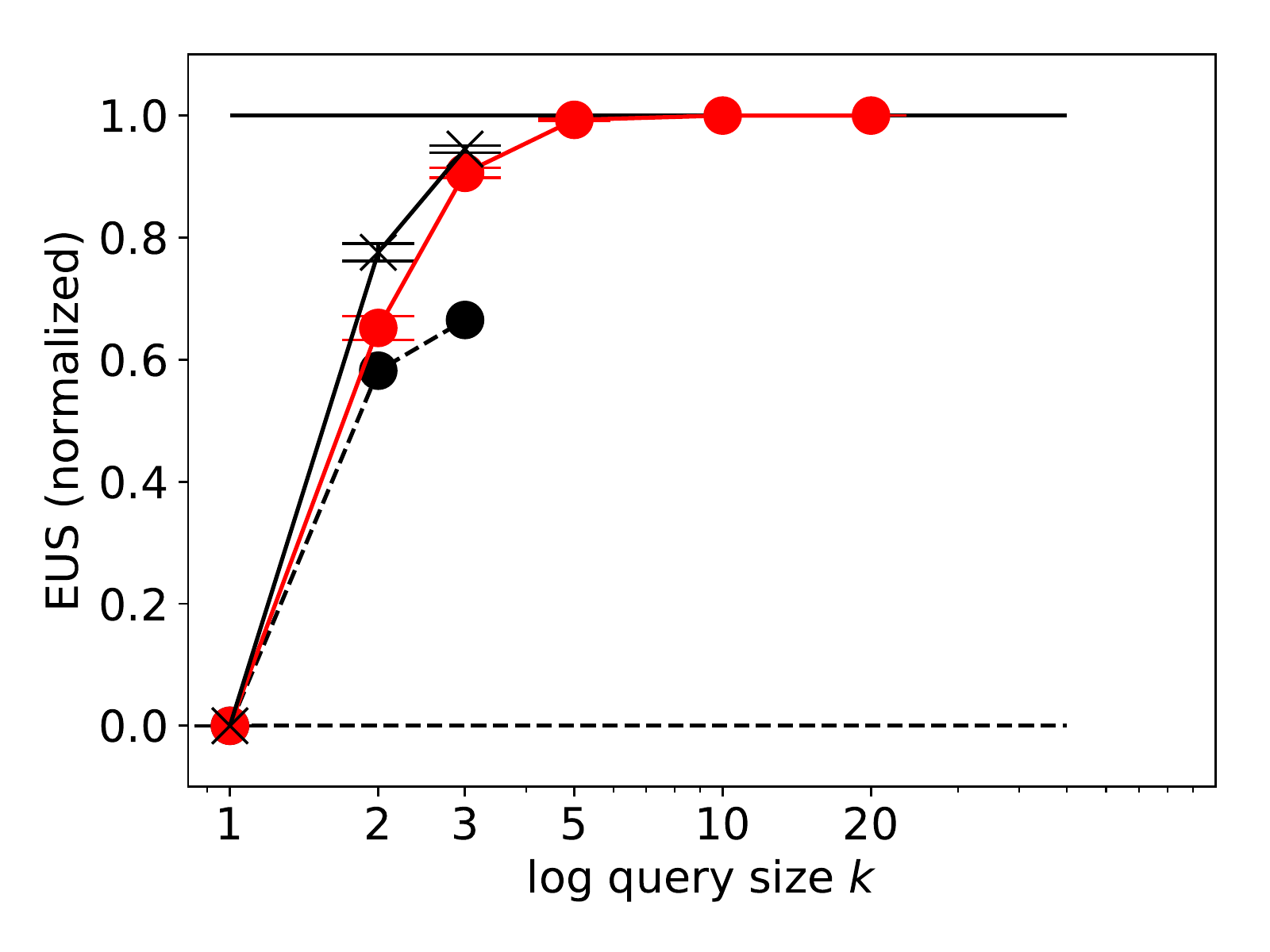}
  \caption{Random Prior, $N=50$}
  \label{fig:random_prior_N50}
\end{subfigure}

\begin{subfigure}[b]{.23\textwidth}
  \centering
  \includegraphics[width=\linewidth]{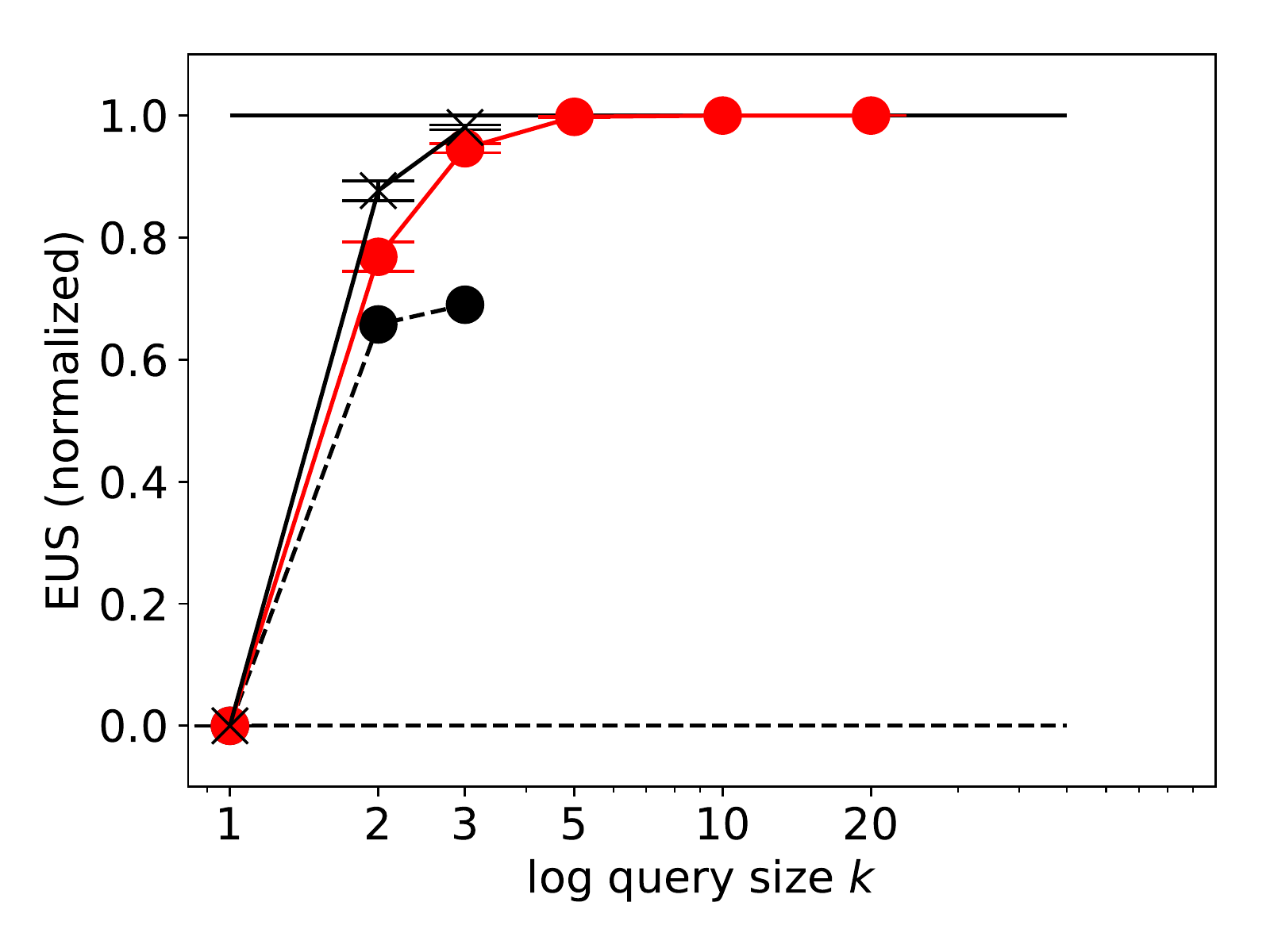}
  \caption{Gaussian Prior, $N=10$}
  \label{fig:gaussian_prior_N10}
\end{subfigure}
\begin{subfigure}[b]{.23\textwidth}
  \centering
  \includegraphics[width=\linewidth]{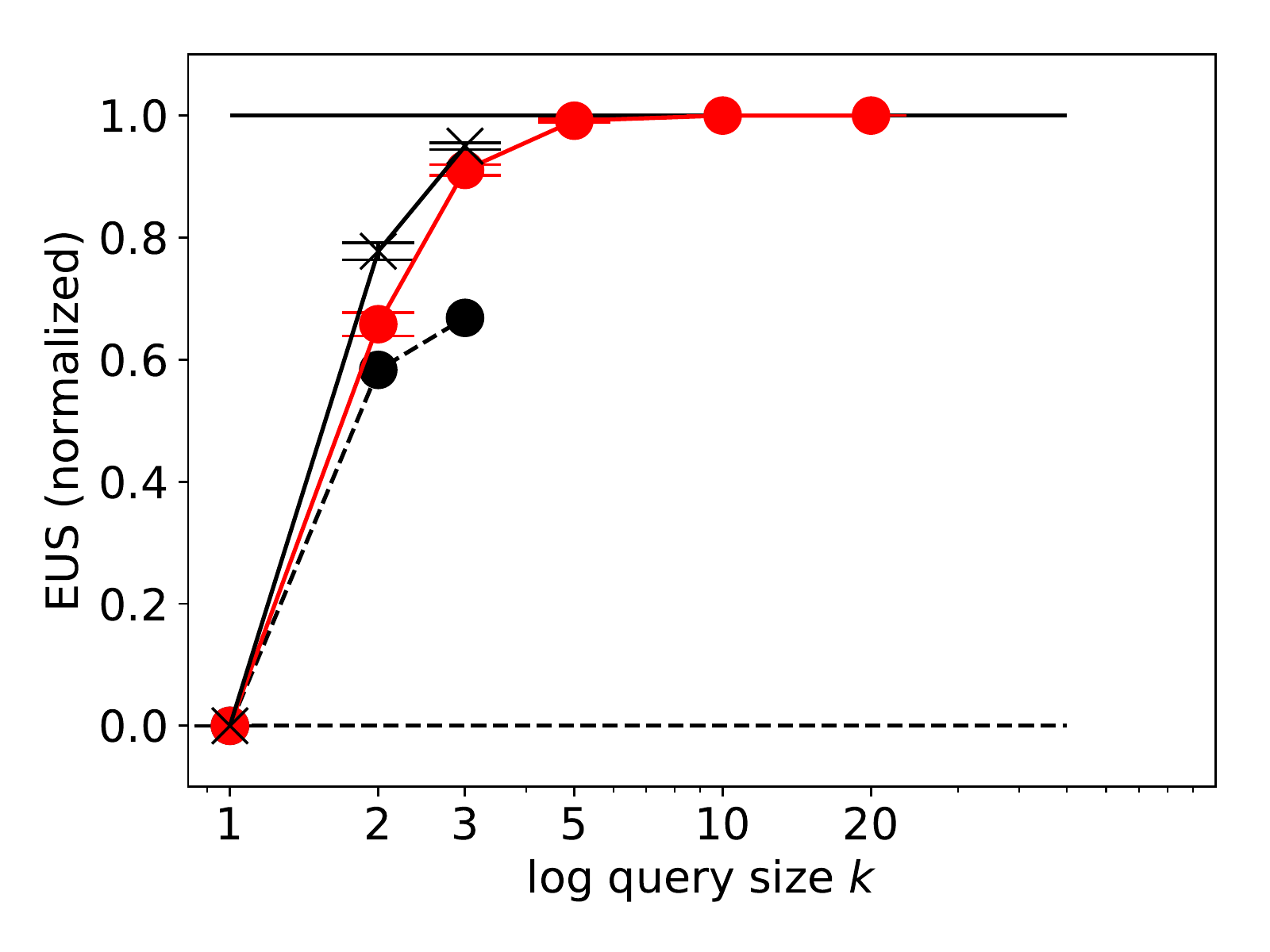}
  \caption{Gaussian Prior, $N=50$}
  \label{fig:gaussian_prior_N50}
\end{subfigure}

\caption
{EUS of the greedy query for the uniform (top), random (middle), and Gaussian (bottom) priors.
The queries are formed from the breakpoints discretization.
The results are means and standard errors of the EUS over 50 problem instances, each consisting of one provider MDP and $N$ recipient MDPs randomly generated as described in \ref{sec:Synthetic MDPs}.
The results are for $N=10$ (left) and $N=50$ (right).}
\label{fig:prior_type}
\end{figure}

Besides priors that are synthetically generated, we here also explore priors that naturally emerge in a two-round querying process.
Specifically, the provider's initial prior $\mu_0$ is a random prior over $N$ candidate recipient's MDPs generated as described above.
The provider forms the first greedy query of size $k_0$, updates its prior to $\mu_1$ based on the recipient's response, and then forms the second greedy query of size $k$ for prior $\mu_1$.
We are interested in the quality of the second greedy query for the updated prior $\mu_1$, which emerges from the first round of querying.
Figure \ref{fig:mutli_round} shows the results for $N=50$, $k_0=2$ and $5$, comparing the greedy query with its theoretical lower bound and the optimal query.
Consistent with the results in Figure \ref{fig:prior_type}, the results in Figure \ref{fig:mutli_round} show that the greedy query is effective for the priors that emerge from the first round of querying.

\begin{figure}[ht]
\centering

\begin{subfigure}[b]{.23\textwidth}
  \centering
  \includegraphics[width=\linewidth]{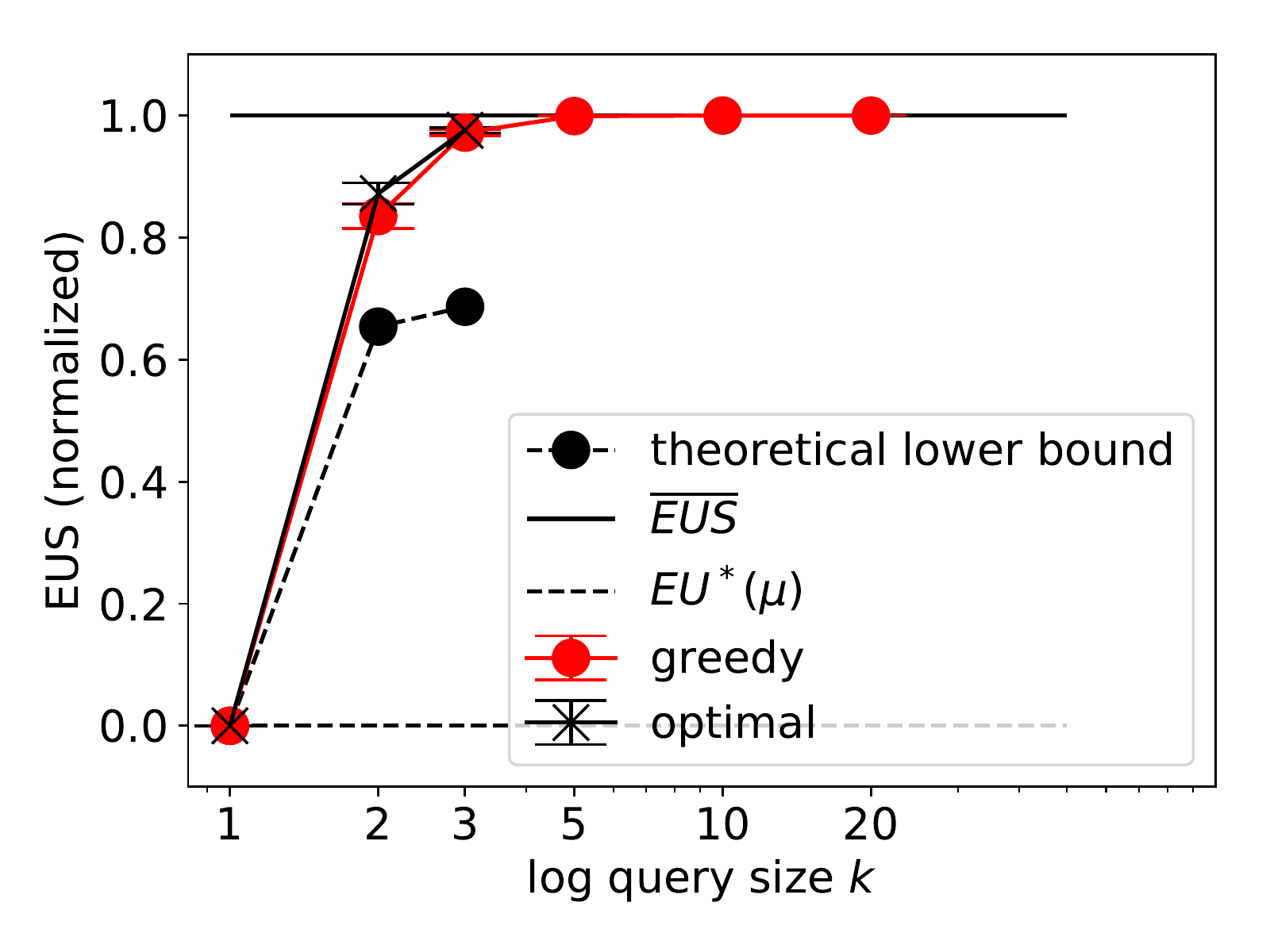}
  \caption{$k_0=2$}
  \label{fig:multi_round_k2}
\end{subfigure}
\begin{subfigure}[b]{.23\textwidth}
  \centering
  \includegraphics[width=\linewidth]{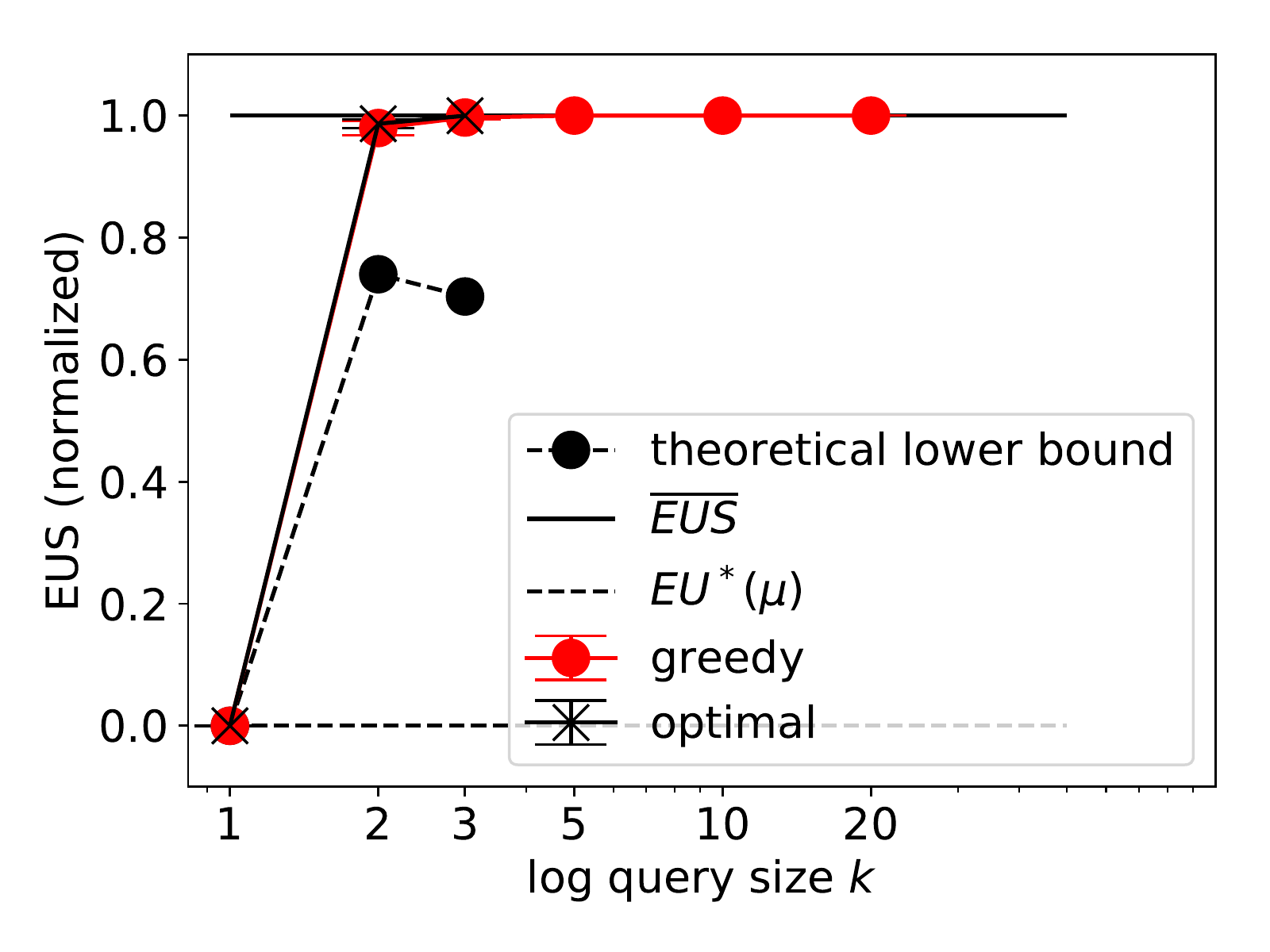}
  \caption{$k_0=5$}
  \label{fig:multi_round_k5}
\end{subfigure}

\caption
{EUS of the greedy query in the second round of querying.
For the first round, the prior is the random prior over $N=50$ candidate recipient's MDPs, and the provider forms the first greedy query of size $k_0$ and updates its prior based on the recipient's response.
For the second round, the provider constructs the second query of size $k$ (X-axis) for the updated prior, and the corresponding normalized EUS is shown along the Y-axis.
The results are means and standard errors of the EUS over 50 problem instances, each consisting of this two-round querying process, for $N=50$ and $k_0=2, 5$.
The provider's MDP and $N=50$ recipient MDPs are randomly generated as described in \ref{sec:Synthetic MDPs}.
}
\label{fig:mutli_round}
\end{figure}

\subsection{Probabilistic Commitment Effectiveness}
\label{appendix:Probabilistic Commitment Effectiveness}

The results in Section \ref{sec:Synthetic MDPs} confirm the effectiveness of our greedy approach to forming a query from the breakpoint commitments. 
Similar to Section \ref{sec:Overcooked}, we measure the joint value of the policies $(\pi^{\rm p}(c^*), \pi^{\rm r}(c^*))$ derived from the optimal commitment found by the centralized algorithm, and the joint value of the optimal joint value when the provider and the recipient plan as a single agent.
Note that this single-agent's optimal value is an upper bound for the joint value of any distributed policies the agents could find.
Table \ref{table:joint value} shows the results for the pairs of the provider's MDP and the recipient's MDP from the same $50$ coordination problems.
For each pair, the values are normalized by the optimal single-agent plan's value.
We also report the joint value achieved by a random policy and by policies $(\pi^{\rm p}(c), \pi^{\rm r}(c))$ derived from a randomly-selected feasible commitment.

As expected, the restriction of coordinating plans only through a commitment specification incurs loss in joint value compared to optimal joint planning, but the loss is modest, given that (as evidenced by the poor performance of the other approaches) the problems were not inherently easy.
The results assure us that our commitment query approach induces high-quality coordination.

\begin{table}[ht]
\centering
\caption{Joint value (mean and standard error) of the policies based on probabilistic commitments.}
\label{table:joint value}
\begin{tabular}{cccc} 
\toprule
\begin{tabular}[c]{@{}c@{}}Random \\Policy \end{tabular} & \begin{tabular}[c]{@{}c@{}}Random \\Commitment \end{tabular} & \begin{tabular}[c]{@{}c@{}}Optimal\\Commitment \end{tabular} & 
\begin{tabular}[c]{@{}c@{}}Optimal\\Single Agent \end{tabular}  \\ 
\midrule
$-4.62 \pm1.58$  & $0.82 \pm 0.03$ & $0.90 \pm 0.02$ & $1.00$\\
\bottomrule
\end{tabular}
\end{table}

\newpage \clearpage
\section{Supplementary Results in Overcooked}
\label{appendix:Supplementary Results in Overcooked}
\subsection{Greedy Query from the Breakpoints}
\label{appendix:Greedy Query from the Breakpoints}
For Overcooked, we repeat the experiments in Section \ref{sec:Synthetic MDPs}, with Figure \ref{fig:overcooked_EUS_runtime_vs_baselines}, Table \ref{table:overcooked discretiztion size}, and Figure \ref{fig:overcooked_EUS_runtime_vs_k} as the counterparts of Figure \ref{fig:EUS_runtime_vs_baselines}, Table \ref{table:discretization size}, Figure \ref{fig:EUS_runtime_vs_k} respectively.

The results in Figure \ref{fig:overcooked_EUS_runtime_vs_baselines} and Table \ref{table:overcooked discretiztion size}, presented in the main body, confirm our conjecture that the breakpoints discretization in Overcooked is relatively smaller, leading to greater efficiency.
Figure \ref{fig:overcooked_EUS_runtime_vs_baselines}(left) confirms that the breakpoints discretization still yields the highest EUS.
Comparing Table \ref{table:overcooked discretiztion size} with Table \ref{table:discretization size}, we see that the breakpoints discretization in Overcooked is even smaller that the even discretization with $n=10$, while the breakpoints discretization is significantly larger than the even discretization with $n=10$ for the synthetic MDPs in Section \ref{sec:Synthetic MDPs}.
Therefore, it is unsurprising to see that the runtime in the Overcooked environment, as shown in Figure \ref{fig:overcooked_EUS_runtime_vs_baselines}(right), is relatively smaller than that in Figure \ref{fig:EUS_runtime_vs_baselines}(right).
Figure \ref{fig:overcooked_EUS_runtime_vs_k} confirms that formulating the greedy query is again both efficient and effective. 

\begin{figure}[ht]
\centering
\begin{subfigure}{.23\textwidth}
  \centering
  \includegraphics[width=\linewidth]{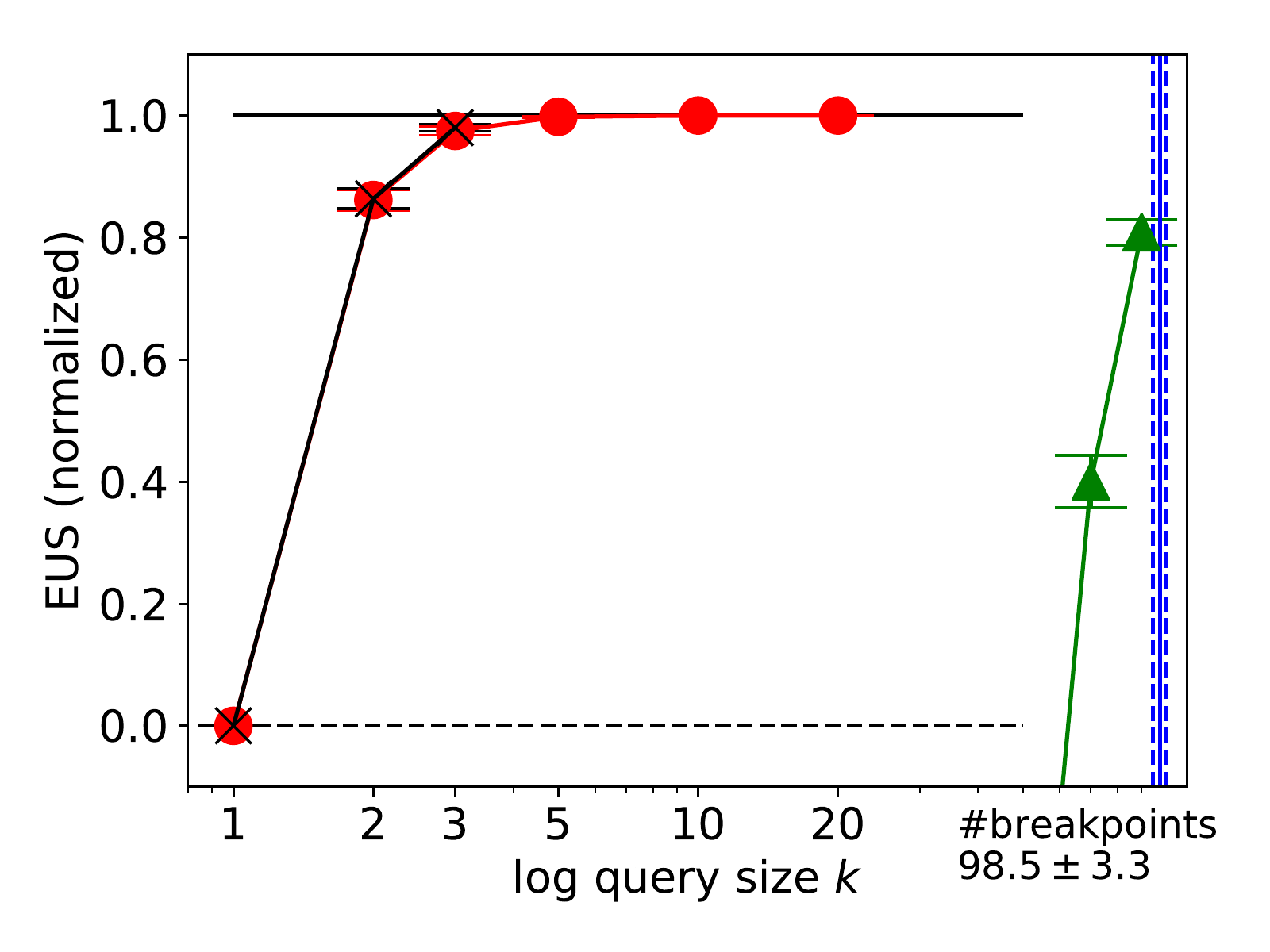}
\end{subfigure}
\begin{subfigure}{.23\textwidth}
  \centering
  \includegraphics[width=\linewidth]{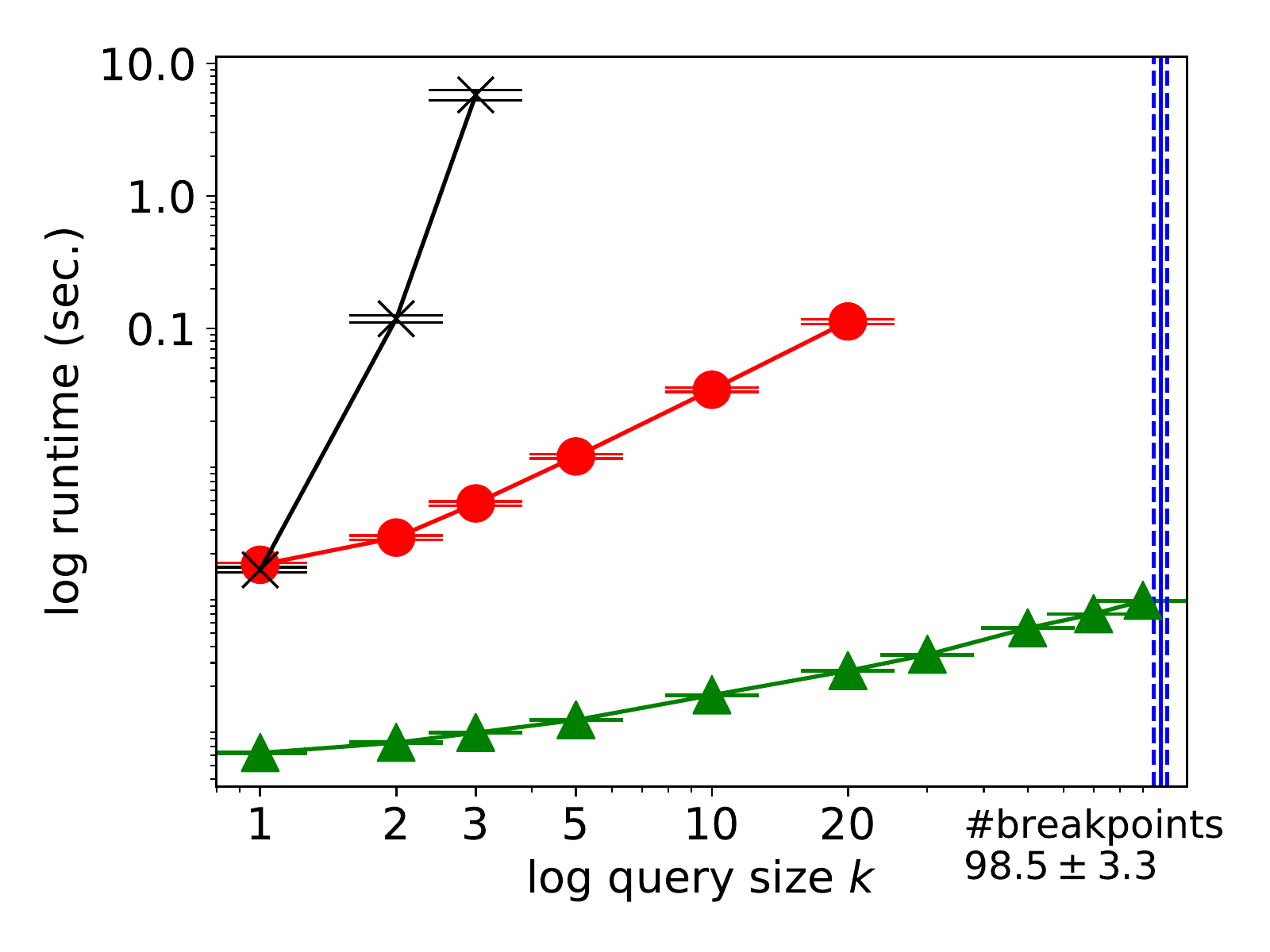}
\end{subfigure}
\caption{ Means (markers) and standard errors (bars) of the EUS (left) and runtime (right) of the optimal, the greedy, and the random queries formulated from the breakpoints in Overcooked.}
\vspace{-3mm}
\label{fig:overcooked_EUS_runtime_vs_k}
\end{figure}

% \begin{table}
% \centering
% \caption{Joint value (mean and standard error) of the policies based on probabilistic commitments in Overcooked.}
% \label{table:overcooked joint value}
% \begin{tabular}{cccc} 
% \toprule
% \begin{tabular}[c]{@{}c@{}} Null~\\Commitment\end{tabular} & \begin{tabular}[c]{@{}c@{}}Random \\Commitment \end{tabular} & \begin{tabular}[c]{@{}c@{}}Optimal\\Commitment \end{tabular} & \begin{tabular}[c]{@{}c@{}}Optimal\\Single Agent \end{tabular}  \\ 
% \midrule
% $0.00$~~                                                   & $0.15 \pm 0.01$~                                             & $0.99 \pm 0.01$~                                             & $1.00$                                                          \\
% \bottomrule
% \end{tabular}
% \end{table}

\subsection{Multiple Food Items}
\label{appendix:Multiple Food Items}
We also consider the scenario where there are more than one food item.
In such a scenario, if the time horizon is long enough, the commitment querying process could consist of multiple rounds, one commitment per round concerned with a single food item.
We here consider evaluating our approach by restricting the chef to make a single commitment regarding one out of the multiple food items, leaving the full-fledged problem of multi-round querying to future work.
If there are $m>1$ food items, besides the commitment time $T_c$ and probability $p_c$, the commitment $c$ should also specify as its commitment feature $u_c$ the id of the food it is concerned with, i.e. $u_c\in\{1,...,m\}$.
It can be easily verified that for each fixed $u_c\in\{1,...,m\}$, the structural properties presented in Section \ref{sec:Structure of the Commitment Space} still hold. Therefore, we can use the binary search procedure to identify the breakpoints for each $u_c\in\{1,...,m\}$ independently, and all the theoretical guarantees presented in the main body still hold.

We repeat the experiments in Section \ref{sec:Overcooked} with $m=3$. The results are presented in 
Figure \ref{fig:overcooked_EUS_runtime_vs_baselines_m3},
Table \ref{table:overcooked discretiztion size m3},
Figure \ref{fig:overcooked_EUS_runtime_vs_k_m3},
and Table \ref{table:overcooked joint value of commitment query m3}
as the counterparts of
Figure \ref{fig:overcooked_EUS_runtime_vs_baselines},
Table \ref{table:overcooked discretiztion size},
Figure \ref{fig:overcooked_EUS_runtime_vs_k},
and Table \ref{table:overcooked joint value of commitment query}, respectively.
These results are qualitatively consistent with the results for $m=1$.

\begin{figure}[ht]
\centering
\begin{subfigure}{.23\textwidth}
  \centering
  \includegraphics[width=1\linewidth]{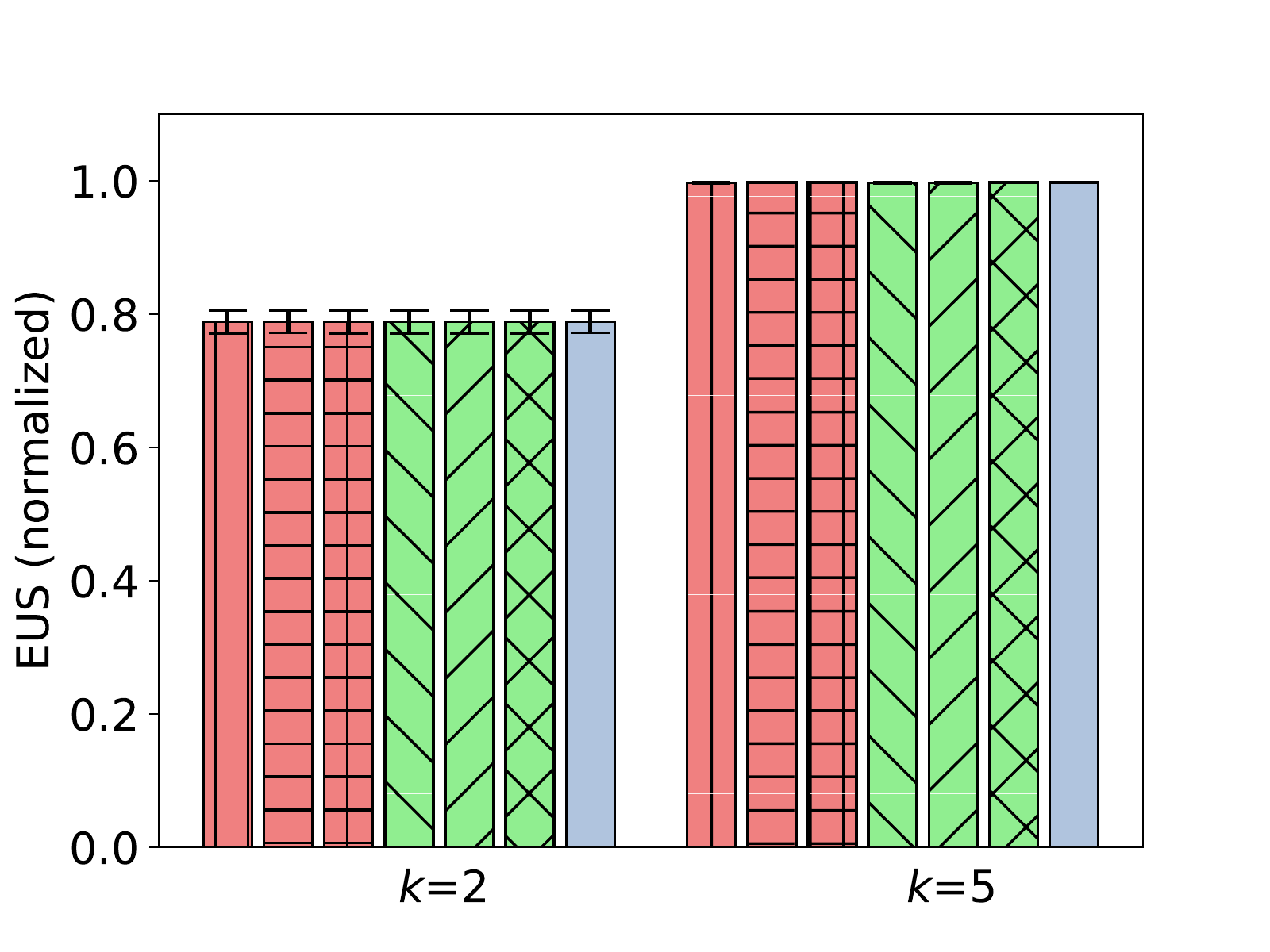}
%   \caption{EUS (normalized)}
%   \label{fig:overcooked_EUS_vs_baselines}
\end{subfigure}
\begin{subfigure}{.23\textwidth}
  \centering
  \includegraphics[width=1\linewidth]{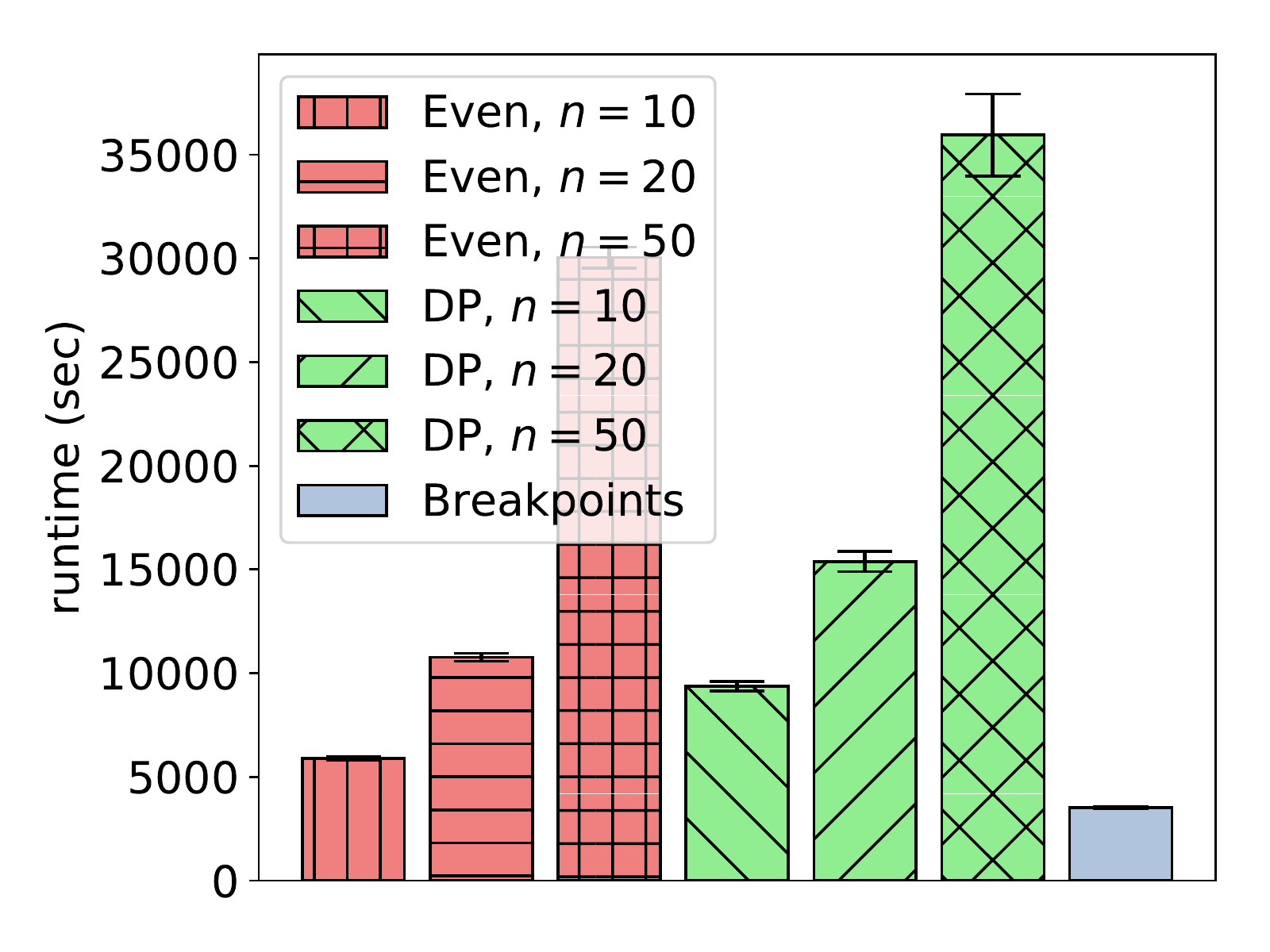}
%   \caption{runtime (sec.)}
%   \label{fig:overcooked_runtime_vs_baselines}
\end{subfigure}
\caption{Means and standard errors of the EUS (left) and runtime (right) 
% of the alternative discretizations 
in Overcooked with $m=3$ food items.
% The runtime is for forming the discretization and evaluating the commitments in the discretization.
}
\vspace{-2mm}
\label{fig:overcooked_EUS_runtime_vs_baselines_m3}
\end{figure}

\begin{table}[ht]
\centering
\caption{Averaged discretization size per commitment time (mean and standard error) in Overcooked with $m=3$ food items.}
\label{table:overcooked discretiztion size m3}
\begin{tabular}{c|ccc} 
\toprule
            & $n=10$        & $n=20$  & $n=50$            \\ 
\hline
Even        & $6.6 \pm 0.1$ & $11.1\pm0.2$ & $28.8\pm0.5$      \\
DP          & $5.3 \pm 0.2$ & $8.7\pm0.4$ & $17.0\pm0.8$       \\ 
\hline
Breakpoints & \multicolumn{3}{c}{$5.0\pm0.1$}  \\
\bottomrule
\end{tabular}
\vspace{-3mm}
\end{table}

\begin{figure}[ht]
\centering
\begin{subfigure}{.23\textwidth}
  \centering
  \includegraphics[width=\linewidth]{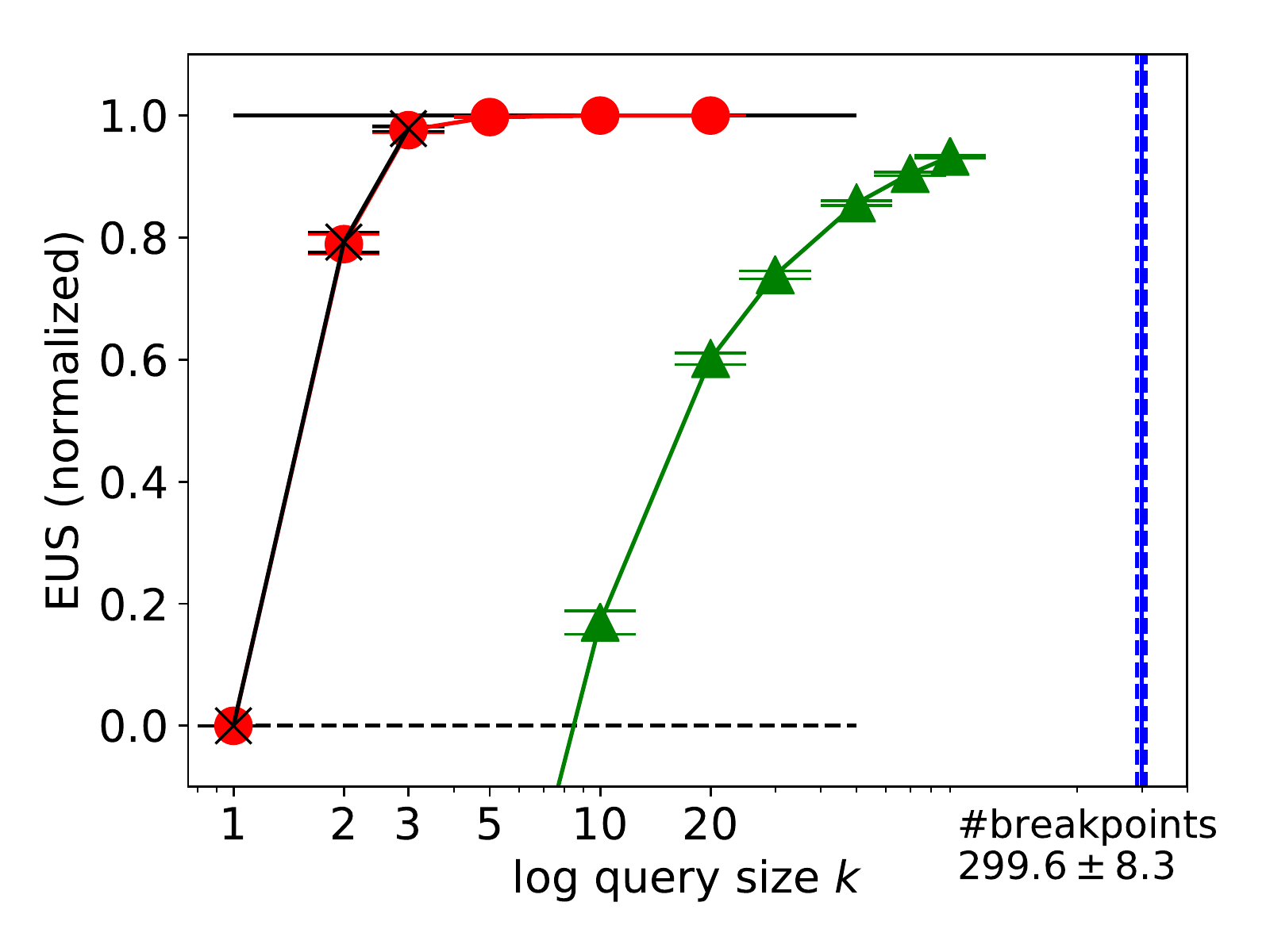}
\end{subfigure}
\begin{subfigure}{.23\textwidth}
  \centering
  \includegraphics[width=\linewidth]{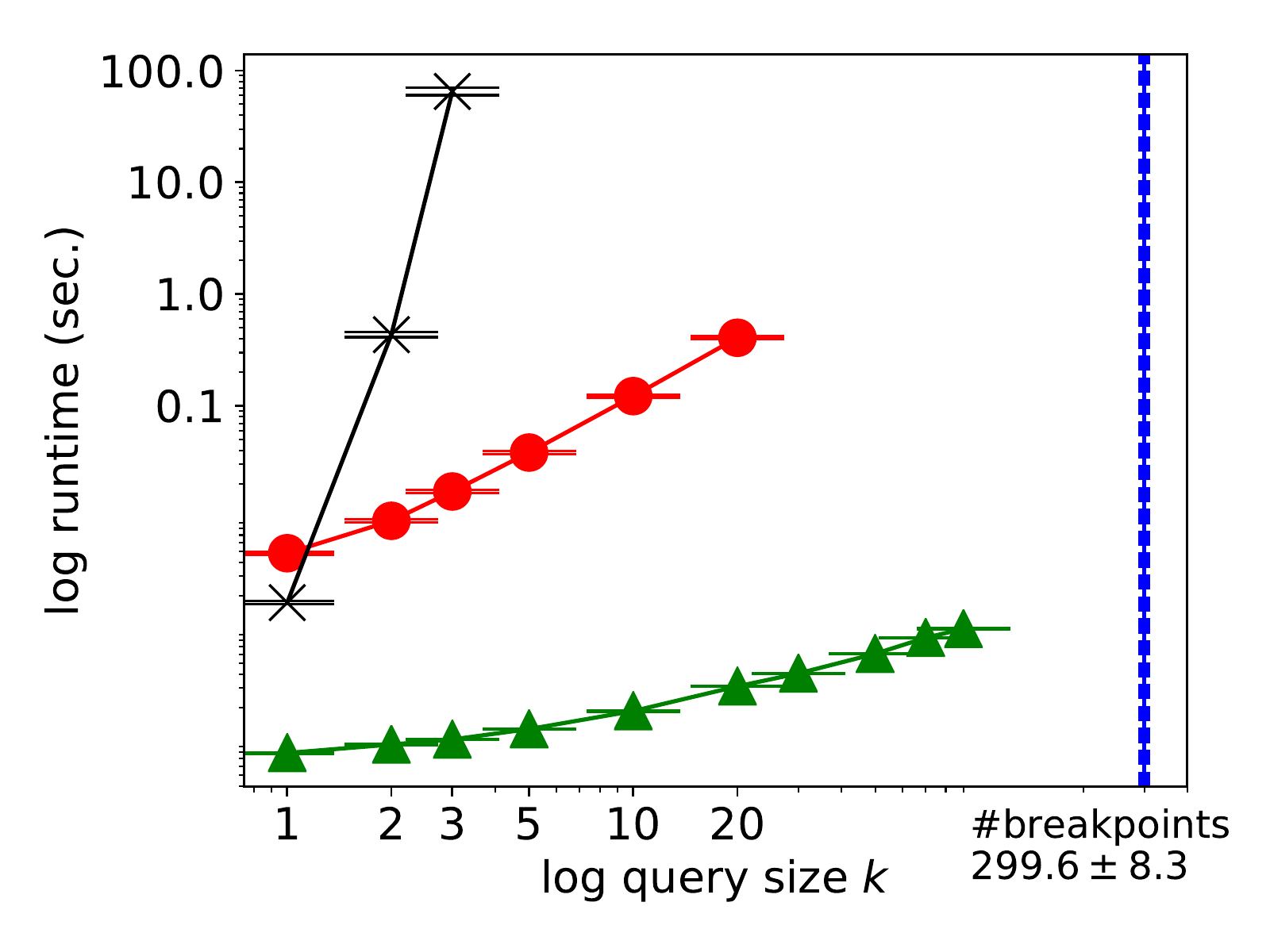}
\end{subfigure}
\caption{ Means (markers) and standard errors (bars) of the EUS (left) and runtime (right) of the optimal, the greedy, and the random queries formulated from the breakpoints in Overcooked with $m=3$ food items.}
\vspace{-3mm}
\label{fig:overcooked_EUS_runtime_vs_k_m3}
\end{figure}

% \begin{table}[t]
% \centering
% \caption{Joint value (mean and standard error) of the policies based on commitment queries in Overcooked with $m=3$ food items.
% }
% \label{table:overcooked joint value of commitment query m3}
% \begin{tabular}{cccc} 
% \toprule
% \begin{tabular}[c]{@{}c@{}}Random~\\Commitment\end{tabular} & \begin{tabular}[c]{@{}c@{}}Query\\ $k = 2$\end{tabular} & \begin{tabular}[c]{@{}c@{}}Query\\$k = 5$\end{tabular} & \begin{tabular}[c]{@{}c@{}}Optimal\\Commitment~ \end{tabular}  \\ 
% \midrule
% $14.5 \pm 1.4 \%$~                                             & $99.6 \pm 0.1\%$~                                      & $99.9 \pm 0.1\%$~                                      & $99.9 \pm 0.1\%$                                                 \\
% \bottomrule
% \end{tabular}
% \end{table}

\begin{table}
\centering
\caption{Values of joint policies (mean and standard error in \%) in Overcooked with $m=3$ food items, with MMDPs normalized to 100 and null commitment to 0.}
\label{table:overcooked joint value of commitment query m3}
\begin{tabular}{l|c|c} 
\toprule
                                 & Centralized                                                & Decentralized                                                         \\ 
\hline
\multicolumn{1}{c|}{Global Obs.} & \begin{tabular}[c]{@{}c@{}}100\\(MMDPs) \end{tabular}      & \begin{tabular}[c]{@{}c@{}} - \\(Wang et al.) \end{tabular}      \\ 
\hline
\multicolumn{1}{c|}{Local Obs.}  & \begin{tabular}[c]{@{}c@{}}$99.9 \pm 0.1$\\(Optimal $c$) \end{tabular} & \begin{tabular}[c]{@{}c@{}}$99.6 \pm 0.1$, $99.9 \pm 0.1$\\(Query $k = 2$, $5$) \end{tabular}  \\ 
\cline{2-3}
                                 & \multicolumn{2}{c}{Null $c$: $0$;~~~~Random $c$: $14.5 \pm 1.4$ }                                                                                               \\
\bottomrule
\end{tabular}

\end{table}

\end{document}